\pdfoutput=1

\documentclass[11pt]{article}

\usepackage[]{acl}

\usepackage{times}
\usepackage{latexsym}

\usepackage{amsmath}
\usepackage{amssymb}
\usepackage{framed} 
\usepackage{subcaption}  
\usepackage{tabularx}  
\usepackage{booktabs}  
\usepackage{multirow}  
\usepackage{makecell}
\makeatletter
\def\new@fontshape{}
\makeatother
\usepackage{siunitx} 
\usepackage{tabularx,booktabs}
\usepackage{xspace}  

\usepackage{tikz}  
\usetikzlibrary{shapes.geometric, arrows}
\usetikzlibrary{calc, positioning}  
\usepackage{forest}  
\useforestlibrary{linguistics}
\forestapplylibrarydefaults{linguistics}

\usepackage{stmaryrd} 

\usepackage[compress,sort,capitalise]{cleveref}  
\usepackage{hyperref}

\usepackage[inline]{enumitem}  

\usepackage{gb4e}  
\noautomath

\crefname{xnumi}{}{}
\creflabelformat{xnumi}{(#2#1#3)}
\crefrangeformat{xnumi}{(#3#1#4)--(#5#2#6)}

\tikzstyle{nnode} = [ellipse, text centered, draw=black, inner sep=2pt] 
\tikzstyle{dnode} = [rectangle, rounded corners, text centered, draw=black] 
\tikzstyle{arrow} = [thick,->,>=stealth]

\definecolor{colorwant}{HTML}{ABBEF1} 
\definecolor{colorgo}{HTML}{DE8BB3} 
\definecolor{colorthe}{HTML}{F1AD83} 
\definecolor{colorboy}{HTML}{FFE19D} 

\pgfdeclarelayer{background}
\pgfdeclarelayer{main}
\pgfsetlayers{background,main}
\newcommand*{\minWidth}{25}  
\newcommand*{\maxValue}{100}
\newcommand{\makeBar}[3]{%
  \tikz[baseline]{
    \node[anchor=base,text width=\minWidth,align=#2,inner sep=0pt,inner xsep=4pt,outer sep=0pt] (n) {\strut{\hfill#1}};  
    \begin{pgfonlayer}{background}
        {
       \edef\color{#3}
       \pgfmathparse{abs(#1/\maxValue)}
       \edef\contents{{\pgfmathresult}}
       \fill[font=\boldmath,color=\color] (n.north west) rectangle ($(n.south west)!{{\contents}}!(n.south east)$);  
       \fill[font=\boldmath,color=cyan!10] ($(n.north west)!{{\contents}}!(n.north east)$) rectangle (n.south east);  
       }
    \end{pgfonlayer}
  }
}
\newcommand{\asbar}[1]{\makeBar{#1}{center}{cyan!25}}

\usepackage{color}
\usepackage{bm}
\definecolor{orange}{rgb}{1,0.5,0}
\definecolor{mdgreen}{rgb}{0.05,0.6,0.05}
\definecolor{mdblue}{rgb}{0,0,0.7}
\definecolor{dkblue}{rgb}{0,0,0.5}
\definecolor{dkgray}{rgb}{0.3,0.3,0.3}
\definecolor{slate}{rgb}{0.25,0.25,0.4}
\definecolor{gray}{rgb}{0.5,0.5,0.5}
\definecolor{ltgray}{rgb}{0.7,0.7,0.7}
\definecolor{purple}{rgb}{0.7,0,1.0}
\definecolor{lavender}{rgb}{0.65,0.55,1.0}
\definecolor{brown}{rgb}{0.6,0.2,0.2}

\newcommand{\code}[1]{}

\newcommand{\hop}{{$k$-hop}}
\newcommand{\hopl}{{$k$-hop}\textsubscript{\textit{large}}}
\newcommand{\hops}{{$k$-hop}\textsubscript{\textit{small}}}
\newcommand{\twohops}{{$2$-hop}\textsubscript{\textit{small}}}
\newcommand{\twohopl}{{$2$-hop}\textsubscript{\textit{large}}}
\newcommand{\threehops}{{$3$-hop}\textsubscript{\textit{small}}}
\newcommand{\threehopl}{{$3$-hop}\textsubscript{\textit{large}}}
\newcommand{\fourhops}{{$4$-hop}\textsubscript{\textit{small}}}
\newcommand{\fourhopl}{{$4$-hop}\textsubscript{\textit{large}}}
\newcommand{\twohop}{{$2$-hop}}
\newcommand{\threehop}{{$3$-hop}}
\newcommand{\fourhop}{{$4$-hop}}

\newcommand{\mathe}{$|\mathcal{E}|$}
\newcommand{\mathr}{$|\mathcal{R}|$}

\renewcommand{\paragraph}[1]{\textbf{#1}}
\newcolumntype{g}{>{\columncolor{orange}}c}

\pgfdeclarelayer{background}
\pgfdeclarelayer{main}
\pgfsetlayers{background,main}

\newcommand{\makedBar}[3]{%
  \tikz[baseline]{
    \node[anchor=base,text width=\minWidth,inner sep=0pt,inner xsep=4pt,outer sep=0pt] (n) {\strut{}};  
    \begin{pgfonlayer}{background}
      \edef\color{#3}%
      \pgfmathparse{abs(#1/\maxValue)}%
      \edef\contents{{\pgfmathresult}}%
      \fill[font=\boldmath,color=\color]
        (n.north west) rectangle ($(n.south west)!{\contents}!(n.south east)$);  
      \fill[font=\boldmath,color=cyan!10]
        ($(n.north west)!{\contents}!(n.north east)$) rectangle (n.south east);  
    \end{pgfonlayer}
  }
}

\newcommand{\dasbar}[1]{\makedBar{#1}{center}{cyan!25}}

\usepackage[T1]{fontenc}

\usepackage[utf8]{inputenc}

\usepackage{microtype}

\usepackage{inconsolata}

\usepackage{amsmath}
\usepackage{amssymb}
\usepackage{mathtools}
\usepackage{amsthm}

\theoremstyle{plain}
\newtheorem{theorem}{Theorem}[section]

\theoremstyle{definition}

\theoremstyle{remark}

\title{Language models can learn implicit multi-hop reasoning,\\but only if they have lots of training data}

\author{Yuekun Yao \\ Saarland University \\  ykyao@coli.uni-saarland.de 
        \And Yupei Du \\ Utrecht University \\ y.du@uu.nl
        \And Dawei Zhu \\ Saarland University \\ dzhu@lsv.uni-saarland.de
        \AND
        Michael Hahn\thanks{\quad Joint senior authors.} \\ Saarland University \\ mhahn@lst.uni-saarland.de
        \And Alexander Koller\footnotemark[1] \\ Saarland University \\ koller@coli.uni-saarland.de
        }

\begin{document}
\maketitle

\begin{abstract}
Implicit reasoning is the ability of a language model to solve multi-hop 
reasoning tasks in a single forward pass, without chain of thought.
We investigate this capability using GPT2-style language models 
trained from scratch on controlled $k$-hop reasoning datasets ($k = 2, 3, 4$). 
We show that while such models can indeed learn implicit $k$-hop reasoning,
the required training data grows exponentially in $k$, and the required
number of transformer layers grows linearly in $k$.
We offer a theoretical explanation for why this depth growth is necessary.
We further find that the data requirement can be mitigated, but not eliminated,
through curriculum learning.
\end{abstract}

\section{Introduction}

Large language models \citep{brown2020language, achiam2023gpt} have demonstrated strong capabilities in complex reasoning tasks \citep{jaech2024openai, guo2025deepseek}. 
With chain-of-thought methods \citep{wei2023chainofthoughtpromptingelicitsreasoning, nye2021workscratchpadsintermediatecomputation}, language models (LMs) learn to explicitly generate the intermediate steps of the given problem before generating the final answer. 
However, such methods incur long inference time \citep{chen2024not} and require costly annotations \citep{nye2021workscratchpadsintermediatecomputation, zelikman2022star}.
This raises the question: \textit{Can language models learn to reason effectively without explicit chain-of-thoughts, i.e., through implicit reasoning?}

\begin{figure}[!t]
    \centering
    \includegraphics[width=1\linewidth]{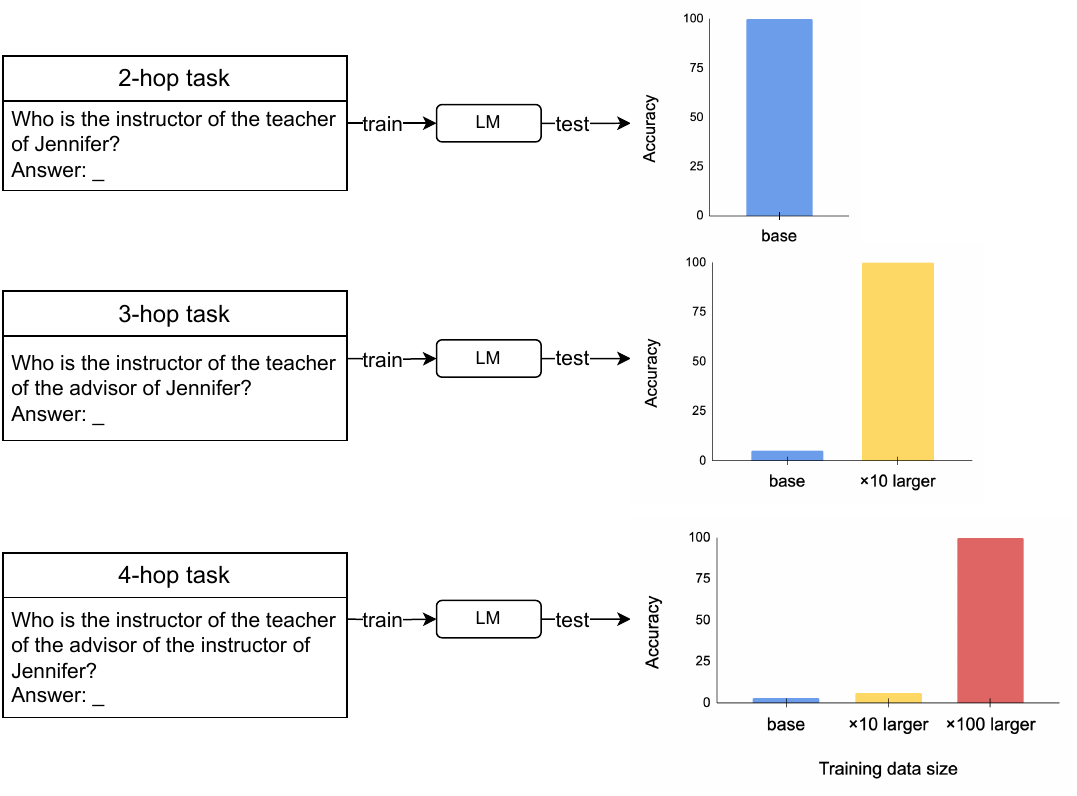}
    \caption{Example illustrating our finding: An LM can be trained to perform implicit $k$-hop reasoning, but requires a large increase in training data as $k$ grows.
    } 
    \label{fig:intro}
\end{figure}

There has been some research exploring implicit reasoning abilities of language models \citep{yang-etal-2024-large-language-models, biran-etal-2024-hopping, wang2024grokking}.
Such studies design their task in a two-hop question answering format, where the model is assumed to know individual facts like \textit{The father of A is B} and \textit{The teacher of B is C}, and then asked questions like \textit{Who is the teacher of the father of A?}. 
Findings from these works suggest that LMs can learn implicit reasoning by combining individual factual knowledge.
However, their reasoning tasks are limited to questions that can be solved with two intermediate steps (i.e.\ 2-hop), leaving more difficult $k$-hop ($k > 2$) reasoning questions alone.
Hence, it remains unclear whether language models can learn to perform such $k$-hop reasoning or not.

In this paper, we study the capacity of language models 
to learn $k$-hop reasoning tasks, where $k=2,3,4$.
By training a randomly initialized GPT2-style transformer \citep{vaswani2017attention, radford2019language} on knowledge (e.g.\ \textit{Jennifer 's instructor is Robert}) and knowledge-based questions (e.g.\ \textit{Who is the instructor of the instructor of Jennifer?}),
we study if such language models learn to generalize to questions that require novel combinations of learned facts.

Our study addresses three research questions. 
\begin{itemize}
    \item First, \textit{can LMs learn implicit $k$-hop reasoning, and if so, under what conditions?}  
Our findings suggest that LMs can indeed learn implicit $k$-hop reasoning, but doing so requires exponentially increasing data budgets as $k$ grows (see Figure \ref{fig:intro}), primarily due to the explosion in the search space of fact combinations.  

    \item Second, we investigate \textit{how models perform $k$-hop reasoning internally} through mechanistic interpretability experiments.  
Our analysis reveals that models trained with sufficient data systematically derive intermediate hop entities in a layer-wise manner, progressing from shallow layers to deeper layers in a step-by-step fashion, consistent with \citet{biran-etal-2024-hopping,wang2024grokking}.
We further show a theoretical lower bound (Theorem~\ref{thm:main-theorem}) suggesting that such a mechanism, with the required depth growing with $k$, may be unavoidable for the transformer architecture.

\item Third, motivated by the substantial data requirements for $k$-hop reasoning, we ask: \textit{How can we reduce the data budget for $k$-hop reasoning?}  
We explore the use of easier ($m$-hop, $m < k$) tasks as auxiliary training signals. 
Our findings show that curriculum learning \citep{elman1993learning,bengio2009curriculum}, which introduces tasks in a progressively harder order, significantly reduces the required training data, while simply mixing $m$-hop tasks with $k$-hop tasks provides only modest gains.
\end{itemize}

Bringing our findings together, we answer the broader question \textit{Can language models learn implicit reasoning?} with a "yes, but" response.  
Language models can solve $k$-hop reasoning; however, this capability comes at the cost of an exponential increase in training data and at least linear growth in model depth as $k$ increases.  
Curriculum learning serves as a significantly effective mitigation strategy to reduce the training data requirement, but the data growth issue still persists.

The code and datasets are available online \footnote{\href{https://github.com/ykyaol7/lm_implicit_multihop_reasoning}{github.com/ykyaol7/lm\_implicit\_multihop\_reasoning}}.

\section{Related work}

\paragraph{Implicit reasoning.}
Many works have shown the power of explicit reasoning ability of language models \citep{wei2023chainofthoughtpromptingelicitsreasoning, saparov2022language, jaech2024openai}.
However, such powerful models, even after heavy pretraining \citep{achiam2023gpt}, generally come with negative results on implicit reasoning tasks \citep{press-etal-2023-measuring, dziri2023faith}.
Relevant studies can mainly be categorized into two groups according to the evaluation task: knowledge-based reasoning \citep{kassner-etal-2020-pretrained, press-etal-2023-measuring, yang2024large}, and mathematical reasoning \citep{nanda2023progress, stolfo-etal-2023-mechanistic}.
In this paper, we study the former task, and we show that GPT2-style language models, are indeed capable of multi-hop reasoning in the cost of training data requirements.

Most previous work studies knowledge-based reasoning with existing large language models \citep{yang2024large, biran-etal-2024-hopping, press-etal-2023-measuring}, where language models are assumed to gain single-hop knowledge through pretraining and evaluated on multi-hop tasks.
Our work instead trains language models on synthetic datasets, which allows us to accurately attribute the model behavior to particular aspects like data and models.
\citet{wang2024grokking} also train a transformer on synthetic datasets to evaluate \twohop\ reasoning.
By contrast, we investigate this question across increasingly complex tasks (e.g.\ $2,3,4$-hop), and we shed light on possible methods that can help under such challenging cases.

\paragraph{Memorization and generalization.}
To train a language model to fit a training set, the model could either memorize all training instances (i.e.\ overfitting), or develop a generalizable solution that solves the test set. 
Previous work studies this in terms of {grokking} phenomenon \citep{power2022grokking, murty-etal-2023-grokking}.
Their findings suggest that both memorized and generalizable solutions exist as neural circuits in the learning process, and increasing training set size encourages the efficient one (i.e.\ generalizable solution) through weight decay \citep{nanda2023progress,varma2023explaining, zhu2024critical}.  
Compared to these work, our study suggests that training data size needs to exponentially grow according to the task difficulty, which provides a possible explanation for the failure of LLMs on complex implicit reasoning tasks.

\section{Dataset} \label{sec:dataset}
We introduce a $k$-hop reasoning dataset we created to train and evaluate LMs in this section. 
We focus on knowledge-based multi-hop reasoning, where generating the correct answer requires combining multiple known facts. 
Following previous work \citep{wang2024grokking,allen2024physics}, we generate datasets according to synthetic knowledge, which allows better control of the task difficulty and attribution of model behaviors. 

\subsection{Task description}
\begin{figure}[t]
    \centering
    \includegraphics[width=\linewidth]{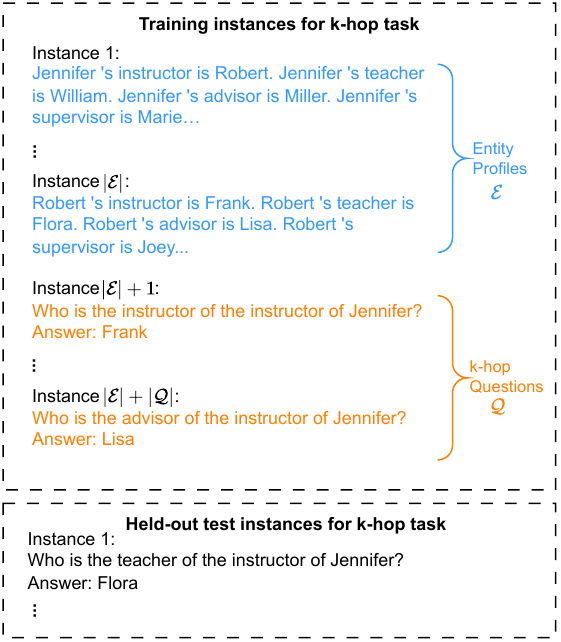}
    \caption{Example of our training and test dataset.
    Here, we use \twohop\ task as an example.}
    \label{fig:dataset_example}
\end{figure}

\paragraph{Definitions.}
The knowledge-based reasoning task includes two main aspects: facts and queries. 
Following prior definitions \citep{yang-etal-2024-large-language-models, wang2024grokking}, we represent a fact as a triple $(e, r, e')$, where $e$ is the subject entity, $r$ is a relation, and $e'$ is the object entity. 
Each relation $r$ acts as a function mapping a subject to an object: $r(e) \to e'$.

A $k$-hop query corresponds to the composition of $k$ such functions, formalized as $r_k(r_{k-1}(\dots r_1(e)))$. 
Answering this query requires reasoning over a chain of $k$ facts: $(e_1, r_1, e_1')$, $(e_1', r_2, e_2')$, \dots, $(e_{k-1}', r_k, e_k')$. 
The intermediate entities $e_1', e_2', \dots, e_{k-1}'$ are referred to as \textit{bridge entities}.
In a $k$-hop query, we refer to the components \((e_1', r_1)\), \((e_2', r_2)\), and so on as the 1-hop, 2-hop, and subsequent hops, respectively.
We thus call $e_1'$ the $1$-hop entity, and $r_1$ the $1$-hop relation.
While prior work has mostly focused on 2-hop queries involving a single bridge entity, we construct datasets for $k \in \{2, 3, 4\}$ to assess models' ability to handle increasingly complex reasoning chains.

\paragraph{Dataset format.}
We create one dataset for each $k \in \{2, 3, 4\}$ task.
Our dataset includes two components: (1) entity profiles encoding known facts, and (2) reasoning questions that query compositions of facts in natural language (see Figure \ref{fig:dataset_example}).
\begin{itemize} 
    \item An entity profile encodes all possible facts for a particular entity where the entity serves as the subject entity (e.g.\ \textit{Jennifer 's instructor is Robert, Jennifer 's teacher is William...}). 
    \item The prompt for our reasoning question is as simple as 
    ``\textit{Who is the teacher of the instructor of Jennifer? \textbackslash n Answer: }'', where \textit{instructor, teacher} refer to  relations and \textit{Jennifer} refers to the queried entity.  
\end{itemize}
We introduce details for generating profiles and questions in Section \ref{sec:dataset:gen}. 
To ensure the model has access to all entity profiles,
the training set includes all possible profiles together with randomly selected reasoning questions, and we use the held-out reasoning questions as the test set.

We construct two dataset variants by varying the number of entities (\mathe) and relations (\mathr): a larger dataset with \mathe$=500$, \mathr$=20$ (denoted \hopl), and a smaller one with \mathe$=250$, \mathr$=10$ (denoted \hops).

\subsection{Data generation} \label{sec:dataset:gen}

\begin{figure}[t]
    \centering
    \includegraphics[width=1\linewidth]{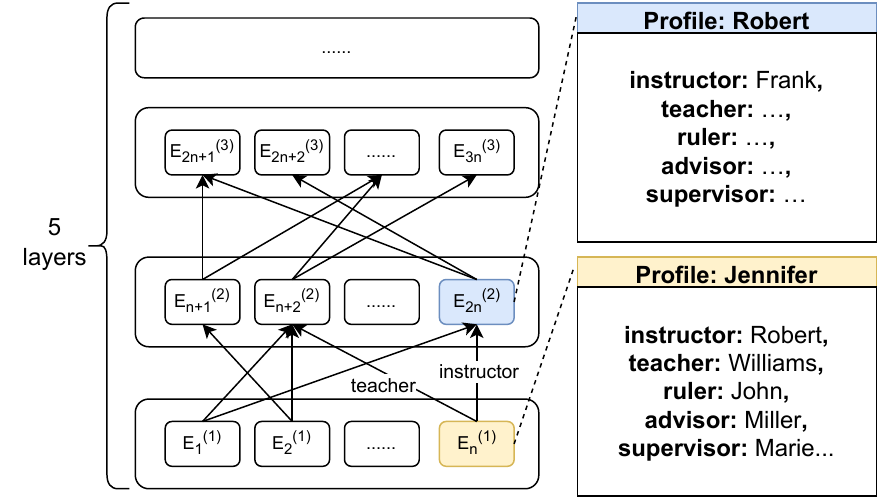}
    \caption{Profile sampling process. We always use 5 layers, and hence $n=$\mathe$/5$ (e.g.\ $100$ for \hopl).}
    \label{fig:datageneration}
\end{figure}

\paragraph{Profile sampling.}
We use the same set of entity profiles across $k = 2, 3, 4$ tasks to ensure fair comparison.
Figure \ref{fig:datageneration} illustrates the process to generate profiles. 
We first sample \mathe\ entity names (e.g.\ \textit{Jennifer}) from a predefined namespace and group them into $K$ disjoint hierarchical layers, where $K$ is the largest $k+1$ value. 
Since we consider $k < 5$, $K$ is fixed at 5. 
Each entity is then linked to \mathr\ randomly selected entities in the upper layer through distinct relations, with relation names reused across layers for generality. 
This structure guarantees that a composition of $k \in \{2, 3, 4\}$ relations starting from any entity in the bottom layer leads to a well-defined target entity.
More details are provided in Appendix \ref{app:dataset:baseline}.  

\begin{table*}[t]
  \centering
  \small
  \setlength{\tabcolsep}{4pt}
  \begin{tabular}{@{}ll|ccccccc@{}}
    \toprule
    & & \multicolumn{7}{c}{Training data budget} \\
    \multicolumn{2}{c}{Dataset}  & ×1 & ×2 & ×5 & ×10 & ×20 & ×50 & ×100 \\
    \midrule
    \multirow{3}{*}{\hops}
      & 2-hop & \asbar{99.8} & \dasbar{100} & \dasbar{100} & \dasbar{100} & \dasbar{100} & \dasbar{100} & \dasbar{100} \\
      & 3-hop & \asbar{5.7}  & \asbar{12.6} & \asbar{99.9} & \asbar{100} & \dasbar{100} & \dasbar{100} & \dasbar{100} \\
      & 4-hop & \asbar{4.3}  & \asbar{6.2}  & \asbar{6.7}    & \asbar{9.2}  & \asbar{96.4}   & \asbar{100}  & \asbar{100}  \\
    \midrule
    \multirow{3}{*}{\hopl}
      & 2-hop & \asbar{99.9} & \dasbar{100} & \dasbar{100} & \dasbar{100} & \dasbar{100} & \dasbar{100} & \dasbar{100} \\
      & 3-hop & \asbar{2.5}  & \asbar{3.1}    & \asbar{4.9}  & \asbar{94.6} & \asbar{100}  & \dasbar{100}  & \dasbar{100}  \\
      & 4-hop & \asbar{2.0}  & \asbar{2.6}  & \asbar{3.1}  & \asbar{3.7}  & \asbar{4.0}  & \asbar{6.3}  & \asbar{100}  \\
    \bottomrule
  \end{tabular}
  \caption{Accuracy of GPT-2 on \hops\ and \hopl\ datasets with different training data budgets. Empty cells indicate that the data budget exceeds the number of possible questions.}
  \label{tab:baseline}
\end{table*}

\paragraph{Profile and question generation.}
Each fact $(e, r, e')$ is mapped to a natural language sentence using a simple template (e.g., \textit{\{subj\}'s \{relation\} is \{obj\}}). 
Following previous work \citep{allen2024physics}, all facts about a given subject entity are concatenated into a single paragraph to form that entity’s profile.
To construct reasoning questions, we sample entities from the bottom layer of the hierarchy (Figure \ref{fig:datageneration}) and recursively traverse $k$ relations to identify the correct answer. 
All valid $k$-hop queries are generated for each source entity.
For example, for $2$-hop queries on \hopl, we can generate up to \mathr$^2 \times$ \mathe$/5 = 40000$ instances.

\section{LMs can learn \hop\ reasoning, but at a large data cost} \label{sec:e1}
Our first objective is to establish that language models can learn implicit \hop\ reasoning, but this requires the number of training instances (i.e.\ \hop\ reasoning questions) grows exponentially as $k$ increases. 
In this section, we empirically demonstrate this by training models on our \hop\ datasets with $k=2,3,4$.

\subsection{Experiment setup} \label{sec:exp_setup}
\paragraph{Model.}
We adopt the smallest GPT-2 architecture \citep{radford2019language} as our model. 
Following recent studies \citep{allen2024physics}, we replace the original positional embeddings in GPT-2 with Rotary Position Embedding (RoPE) \citep{su2024roformer}.
We use the GPT-2 tokenizer \citep{radford2019language} and extend its vocabulary by adding all possible entity names from our dataset.
The training objective is the causal language modeling loss calculated over all tokens in each prompt.
In our main experiments, we train the model from scratch by randomly initializing all parameters. 
Additionally, we conduct experiments using the pretrained GPT-2 and its larger variants (see Appendix \ref{app:scaleup} for results).

\paragraph{Training.}
We set the training steps to 20k for all tasks except \fourhopl, where we extend the training to 40k steps to ensure convergence.
We apply a cosine learning rate scheduler with 1k warm-up steps.
Each experiment is repeated across three runs using different random seeds, and we report the average performance.
Details of hyperparameters for model architecture and training are provided in Appendix \ref{app:training}.

\paragraph{Dataset.}
We utilize the \hops\ and \hopl\ datasets introduced in Section \ref{sec:dataset} for training and evaluation, considering $k=2,3,4$.
This results in six datasets in total.
For the \twohop\ task, we generate all possible reasoning questions and randomly sample 50\% for the \twohopl\ training set and 80\% for \twohops.
All entity profiles are included in the training sets.
The test set consists of 3,000 instances randomly selected from the held-out questions, except for \twohops, which contains only 1,000 held-out questions.
We report the details and statistics of our datasets in Appendix \ref{app:dataset:baseline}.

For \threehop\ and \fourhop\ tasks, we find that the same data size as the \twohop\ training set results in random guessing performance.
Thus, we progressively increase the training data size by defining the base training budget $bg$ as the number of reasoning questions in the \twohop\ training set.
We create training sets by scaling $bg$ with ratios from the set $\{ \times 1, \times 2, \times 5, \times 10, \times 20, \times 50, \times 100 \}$.
For each ratio $r$, we randomly sample $r \times bg$ reasoning questions for training.
The test set for each $k$-hop task always includes 3,000 instances randomly sampled held-out instances except for \twohops.

\paragraph{Evaluation.} 
For each test instance,
we provide the language model with the prompt up to the answer token (e.g., ``\textit{Who is the instructor of the instructor of Jennifer? \textbackslash n Answer: }'') and evaluate the accuracy of the generated token against the gold answer.
Greedy decoding is used for evaluation.

\subsection{Results} \label{sec:base_results}

\begin{figure*}[!t] 
  \centering
  \begin{subfigure}[b]{0.48\textwidth}
    \includegraphics[width=\linewidth]{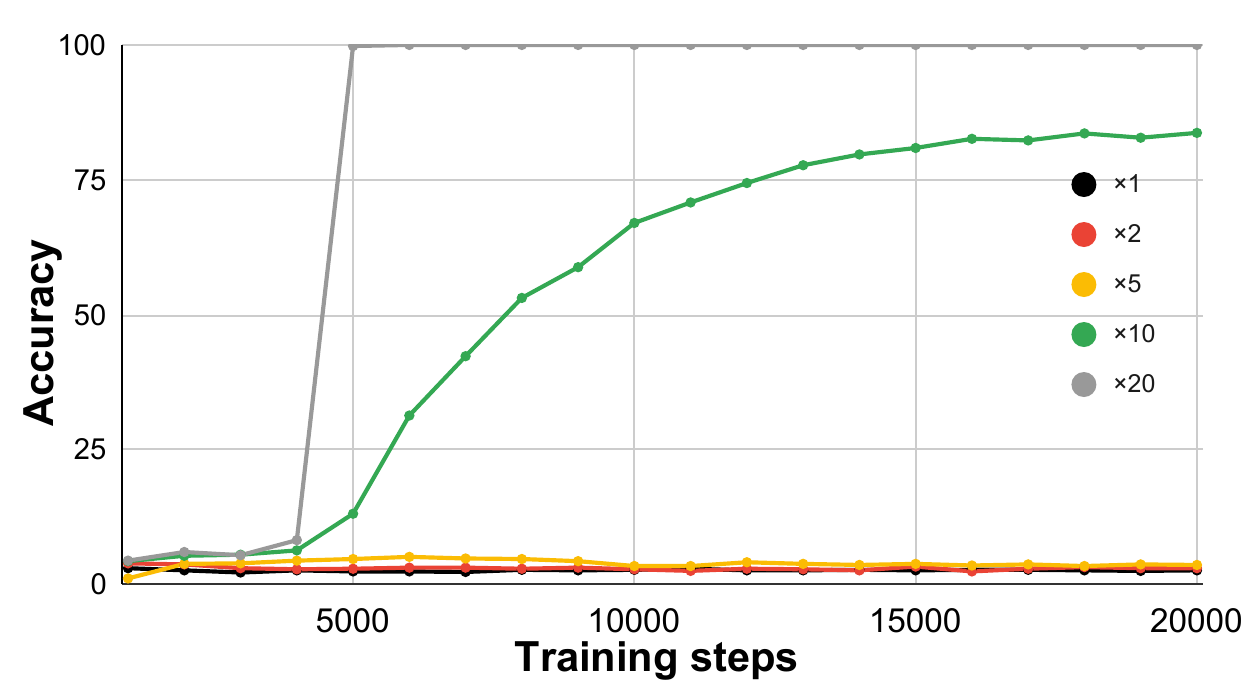}
    \caption{\threehopl}
    \label{fig:hop3_large}
  \end{subfigure}
  \hfill
  \begin{subfigure}[b]{0.48\textwidth}
    \includegraphics[width=\linewidth]{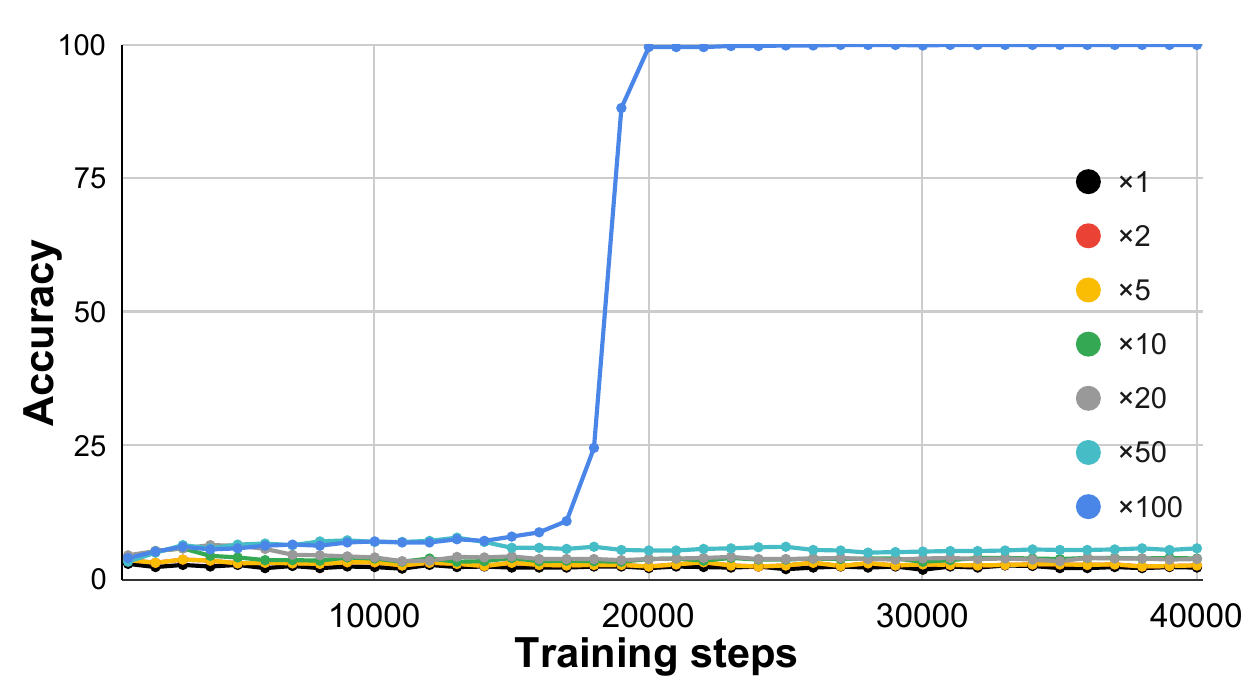}
    \caption{\fourhopl}
    \label{fig:hop4_large}
  \end{subfigure}
  \caption{Model accuracy on \threehopl\ and \fourhopl. 
    \threehopl\ can only be solved when the training budget is increased by at least a factor of $\times10$, while \fourhopl\ requires $\times100$. 
    Using $\times 20$ budget further encourages convergence on \threehopl\ compared to $\times10$ budget.
  }
  \label{fig:large_model_comparison}
\end{figure*}


\paragraph{Language models can learn $k$-hop reasoning.} 
Table \ref{tab:baseline} reports the test accuracy of our models under varying training data budgets.
Our first observation is that GPT-2 models are capable of achieving 100\% accuracy not only on \twohop\ tasks but also on more complex \threehop\ and \fourhop\ tasks, given a sufficiently large training data budget with the same $k$ as the test set. 
This is a significant finding, as each entity profile appears individually in the training set without any explicit instructions on how to combine them to solve multi-hop tasks.
The perfect accuracy suggests that language models can learn the underlying reasoning process based solely on input-output pairs, even without explicit rationales.

\paragraph{However, data requirements increase exponentially with $k$.}
We further observe that the base training data budget ($\times1$) is insufficient for the model to effectively learn \threehop\ and \fourhop\ tasks, as evidenced by test accuracy below $10\%$. 
As the training data budget increases, model performance improves correspondingly. 
We define a model as successfully learning the task if it achieves a test accuracy above $80\%$. 
On \hops\ datasets, a minimum budget of $\times5$ is necessary to learn the $3$-hop task, whereas the $4$-hop task requires a budget of at least $\times20$.  
On \hopl\ datasets, the data budget required for the $3$-hop task is $\times10$, and for the $4$-hop task, it escalates to $\times100$. 
These findings suggest that the training data budget grows in an exponential manner as the value of $k$ increases.

We also plot the test accuracy of one training run on \hopl\ across training steps in Figure \ref{fig:large_model_comparison} (For \hops\ results see Appendix \ref{app:detailed_results:hops}). 
The plots show that a larger training budget not only results in higher accuracy but also accelerates model convergence.
For instance, in Figure \ref{fig:hop3_large}, the $\times20$ budget reaches 100\% accuracy by step $5000$, while the $\times10$ budget only achieves $10\%$ accuracy at the same step. 
This finding is also consistent with \citet{wang2024grokking} reported in \twohop\ reasoning tasks. 
We extend these observations by demonstrating that the data budget becomes even more critical as the complexity of the reasoning task increases. 

\subsection{Why data-hungry?}

Results so far highlight the substantial data requirements for $k$-hop  tasks, but the reason for this remains unclear. 
Increasing the value of $k$ leads to both an increase in the number of combined facts (i.e., $k$ facts for each entity) and a corresponding exponential increase of the search space (i.e., \mathr$^k$ relation combinations per entity). 
Our objective here is to disentangle the effects of these two factors and identify the primary source of data inefficiency.

\begin{figure}[t]
    \centering
    \includegraphics[width=\linewidth]{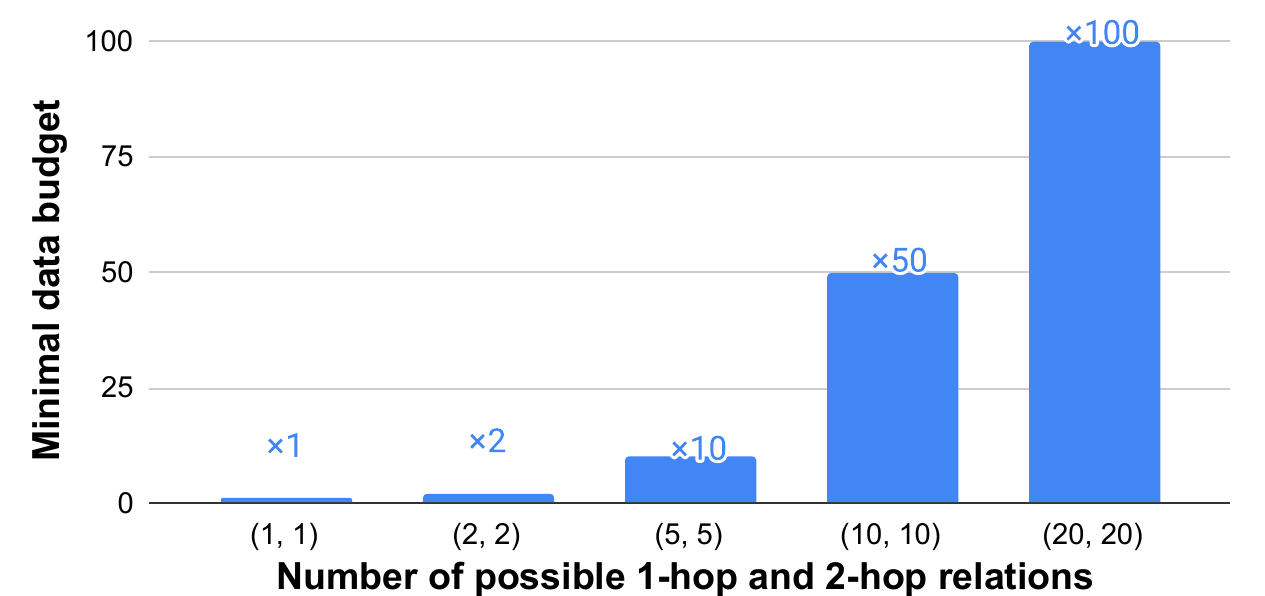}
    \caption{Case study on \fourhopl. 
    x-axis denotes the number of $1$-hop and $2$-hop relations, e.g.\ $(1,1)$ denotes that $1$-hop and $2$-hop relations are fixed across all $4$-hop questions.
    Fixing $1$-hop and $2$-hop relations reduces the required training budget to $\times1$, while increasing them leads to rapid budget growth. 
    }
    \label{fig:hop4_case}
\end{figure}

\paragraph{Setup.}
To investigate this question, we conduct a case study on the \fourhopl\ dataset, where we vary the number of $1$-hop and $2$-hop relations while holding the number of relations in the $3$-hop and $4$-hop positions constant.  
In the original dataset, each hop position can take one of \mathr$=20$ possible relations. 
For this study, we generate new training and test sets by limiting the number of $1$-hop and $2$-hop relations to values from the set $\{ 1, 2, 5, 10, 20 \}$. 
For each configuration, we train GPT-2 models and determine the minimal data budget required to achieve $80\%$ test accuracy.  

Figure \ref{fig:hop4_case} presents the results. 
We observe that when the number of $1$-hop and $2$-hop relations is restricted to a single relation, the model can successfully learn \fourhop\ task using the base data budget. 
However, as the number of relations increases, the required data budget rapidly increases. 
This result suggests that the main source of data inefficiency in $k$-hop reasoning tasks is the exponential growth in the number of relation combinations, rather than the number of individual facts to be combined.

\section{LMs reason through layer-wise lookup, incurring the cost of depth} \label{sec:e2}
The second objective is to understand the underlying mechanism by which  language models  solve the \hop\ task.
We first demonstrate that language models solve such tasks by layer-wise lookup of bridge entities of a $k$-hop query $r_k(r_{k-1}(\ldots r_1(e)))$ through empirical evidence (e.g.\ mechanistic interpretability).
Building on this finding, we then establish a theoretical lower bound, showing that the model’s depth must grow with $k$ to maintain such layer-wise lookup mechanism.

\subsection{Experiment setup}
We design two experiments to investigate the model's internal reasoning process: probing and causal intervention.
For both experiments, we select the model trained on \fourhopl\ with a $\times100$ budget, as it achieves strong performance.   

\paragraph{Probing.}
We use probing tasks \citep{belinkov-glass-2019-analysis, liu-etal-2019-linguistic} to assess whether information about intermediate bridge entities is encoded in the hidden representations. 
In this setup, we freeze the model parameters and train a linear probe classifier on top of the hidden states to predict the correct entity.
We train one probe classifier for each hop position, predicting the corresponding bridge entity in the query. 
The probe is trained across all transformer layers and all tokens in the prompt to identify where and when information about the bridge entities is encoded.  
We split the \fourhopl\ test set into $80/20\%$ training and evaluation sets for training the probe classifiers.

\paragraph{Causal intervention.}
While probing shows whether information about bridge entities is encoded in the hidden representations, it does not tell us whether the model actually relies on this information to generate the final answer. 
We further design \textit{activation patching} \citep{vig2020investigating, meng2022locating} experiments to investigate it.

The core idea of activation patching is to replace the residual stream (i.e., the output of a residual layer in a transformer block) at a specific layer $L_i$ and prompt token $t_j$, and measure the resulting change in the output probability of the correct answer. 
For convenience, we call this residual stream $res(L_i, t_j)$.
In this section, we focus on the last token position in the input prompt (i.e., $t_j$ always being the last input token, which is whitespace \textit{<space>}), as justified in Section~\ref{sec:e2_results}. 

Suppose we are given a \hop\ test instance and aim to measure the causal effect of $res(L_i, t_j)$, the residual stream at layer $L_i$ and token $t_j$. 
For clarity, we define three types of runs as follows.
\textbf{Clean Run:} The original forward pass of the test instance, producing the output probability of the correct answer as $P_{\text{clean}}$.  
\textbf{Corrupted Run:} A distinct \hop\ instance selected to serve as the source of the patched residual stream.  
\textbf{Patched Run:} The modified run, where the residual stream $res(L_i, t_j)$ in the clean run is replaced with the corresponding $res(L_i, t_j)$ from the corrupted run, leaving other layers unchanged.
The output probability in the patched run is denoted as $P_{\text{patched}}$. 
The causal effect of the targeted residual stream is defined as $P_{\text{clean}} - P_{\text{patched}}$, where a larger effect indicates greater reliance on the removed information. 
We calculate the causal effect for each layer and report the average effect across 3000 held-out instances.

The aim of our intervention experiment is to measure the effect of bridge entity information at different hop positions (e.g., $1$-hop, $2$-hop).
Hence, we define four types of corrupted runs for each clean run: \textit{C}\textsubscript{1-hop}, \textit{C}\textsubscript{2-hop}, \textit{C}\textsubscript{3-hop}, and \textit{C}\textsubscript{4-hop}. 
In a \textit{C}\textsubscript{i-hop} run, we select a corrupted instance where the gold $i$-hop entity differs from the clean run, while the entities of other hop positions remain unchanged. 
This setup allows us to measure the effect of perturbing a specific $i$-hop entity while keeping other bridge entities unchanged.

\subsection{Results} \label{sec:e2_results}

\begin{figure}[t]
    \centering
    \includegraphics[width=1.0\linewidth]{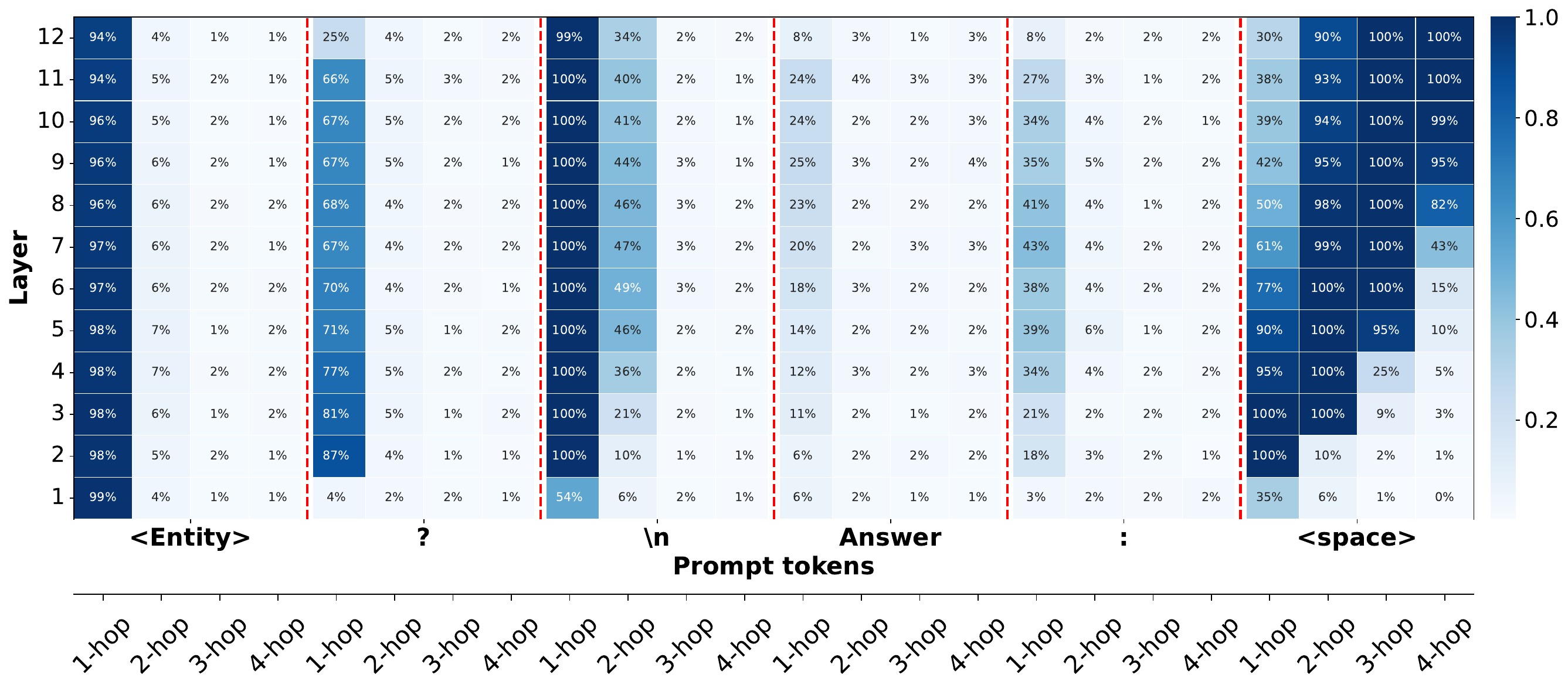}
    \caption{
    Probing results across tokens in the input prompt. Each token is represented by four columns, corresponding to $1$-hop to $4$-hop bridge entities. Note that tokens preceding \textit{<Entity>} cannot include information about any of the entities, and are thus not shown here.
    Only the \textit{<space>} token consistently encodes information about all four bridge entities, indicating that reasoning is concentrated at the last token before answer generation.
    }
    \label{fig:probe}
\end{figure}

\paragraph{Bridge entities are encoded in the last token position.}
Figure \ref{fig:probe} presents the probing results across layers and token positions in the input prompt, e.g., \textit{Who is the instructor of the teacher of the advisor of the instructor of <Entity>? \textbackslash n Answer:<space>}.
We report results only for tokens after the \textit{<Entity>} token, as preceding tokens cannot contain information about the target bridge entities.
Since the vocabulary size of each $i$-hop entity is 100, a random baseline provides $1\%$ accuracy.  

Notably, the hidden representation of the last input token encodes information about all necessary bridge entities for predicting the final answer.
Instead, probe classifiers show low accuracy for other token positions, suggesting that the reasoning process likely occurs in the position immediately before generating the final answer. 
We confirm this by observing zero casual effects on preceding tokens with additional activation patching experiment (see Appendix \ref{app:mi:preceding_tokens}).
We thus focus our causal intervention experiments on this \textit{<space>} token. 

\begin{figure}[t]
  \centering
    \includegraphics[width=.75\linewidth]{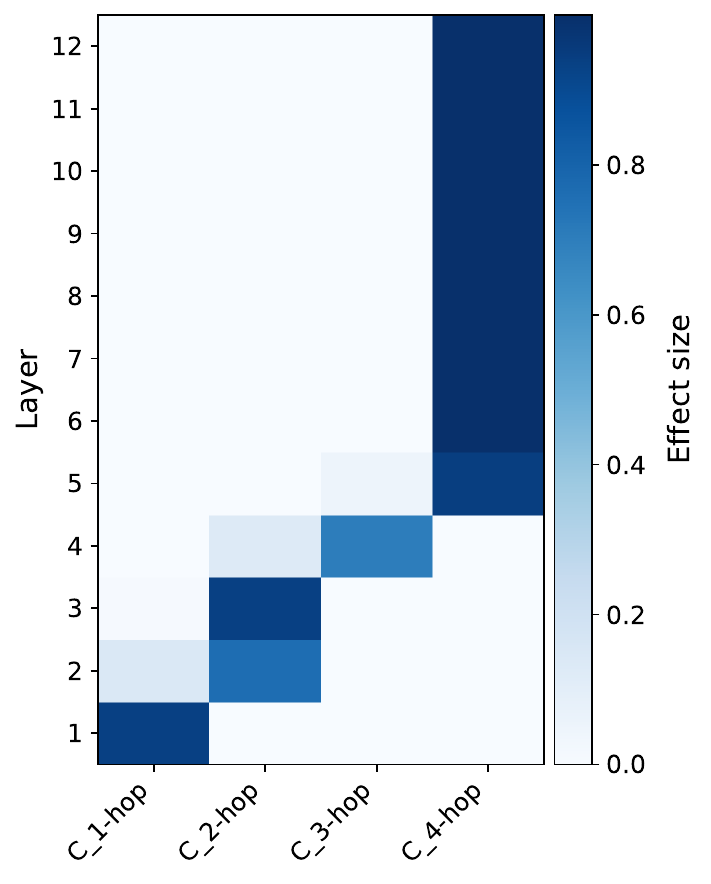}
    \label{fig:actpatch_heatmap}
  \caption{%
    Causal interventions reveal a layer-wise lookup mechanism:
    Intervening on the $1$-hop entity has the strongest effect in the 1st layer, and little effect in higher layers. Intervening on the $2$-hop entity has an effect mainly in the 2nd and 3rd layers; analogously for the $3$-hop and $4$-hop entities. Overall, these results indicate that entities are looked up layer-by-layer. 
  }
  \label{fig:intervention_heatmaps}
\end{figure}

\paragraph{Output prediction relies on bridge entity information.}
Figure \ref{fig:intervention_heatmaps} shows the causal effects across layers in our intervention experiment. 
For each $i$-hop entity, we identify specific layers that the model relies on to generate the final answer. 
Moreover, the model organizes the reasoning process in a layer-wise manner, with shallower layers handling lower-hop entities, and deeper layers handling higher-hop entities.
This layer-wise lookup confirms that the model leverages bridge entity information to perform multi-hop reasoning, which generalizes prior observations from \twohop\ tasks \citep{biran-etal-2024-hopping, wang2024grokking}. 

\subsection{Theoretical Analysis} \label{sec:theoretical_analysis}


We have found that language models perform the $k$-hop task by layer-wise lookup.
This suggests that transformers may need depth \emph{linear} in the number of reasoning steps.
%
Here, we discuss how this result relates to the in-principle expressiveness of transformers.
Formally, we consider a universe $\mathcal{E}$ of entities (e.g., $\{\textit{Jennifer}, \textit{Frank}, \dots\}$) and a set $\mathcal{R}$ of maps $r : \mathcal{E} \rightarrow \mathcal{E}$ (e.g., when $r=\textit{instructor}$, $e=Jennifer$, then $r(e)$ denotes the instructor of Jennifer).
We consider the task of mapping an input string 
\textit{``Who is the $r_k$  of the  $r_{k-1}$ \dots $r_{2}$ of the  $r_1$  of   $e$? Answer:''} (as in Figure~\ref{fig:dataset_example})
to the entity $r_{k}(\dots r_1(e)\dots) \in \mathcal{E}$ (e.g., \textit{instructor(teacher(Jennifer))}). 
We lower-bound the number of layers needed in the case where the \textit{attention pattern does not depend on the query $e$}.
We consider this a reasonable special case, as there is no obvious way in which query-dependent attention would help solve the $k$-hop task.
In this case:
\begin{theorem}[See Appendix~\ref{app:theory} for proof]\label{thm:main-theorem}
    Consider a causal transformer operating in $p$ bits of precision, with $d$ hidden units, $H$ heads and $L$ layers. Assume it performs $k$-hop reasoning over $\mathcal{E}$ and $\mathcal{R}$ as defined above.
    Assume further that the attention pattern does not depend on $e$.
    If $k \leq |\mathcal{E}|-2$, then, for some  $\mathcal{R}$,
    \begin{equation}
    L \geq \frac{k}{8pdH}
\end{equation}
\end{theorem}
We note that there are relation sets $\mathcal{R}$ for which shortcuts with few layers may exist, but the result shows that a linear number is needed in the worst case. 
This statement expresses a \emph{width-depth tradeoff}: the product of the number of layers, bits of precision, width, and number of heads needs to grow linearly in $k$.
In particular, within a single model (i.e., fixing $d$, $H$, and $p$), the analysis predicts that, as $k$ grows, more and more layers need to be involved in the hop-by-hop retrieval, as we found empirically (Figure~\ref{fig:intervention_heatmaps}).
We also note that existing results \citep{chen2024theoretical} are not applicable  to our $k$-hop task (Appendix~\ref{app:theory-discussion}).
Further empirical evidence supports our theoretical prediction (see Appendix \ref{app:depth_effect}).

\section{Curriculum learning mitigates the data requirement, but doesn't solve it} \label{sec:e3}

\begin{figure}[!t] 
  \centering
  \begin{subfigure}[b]{0.46\textwidth}
    \includegraphics[width=\linewidth]{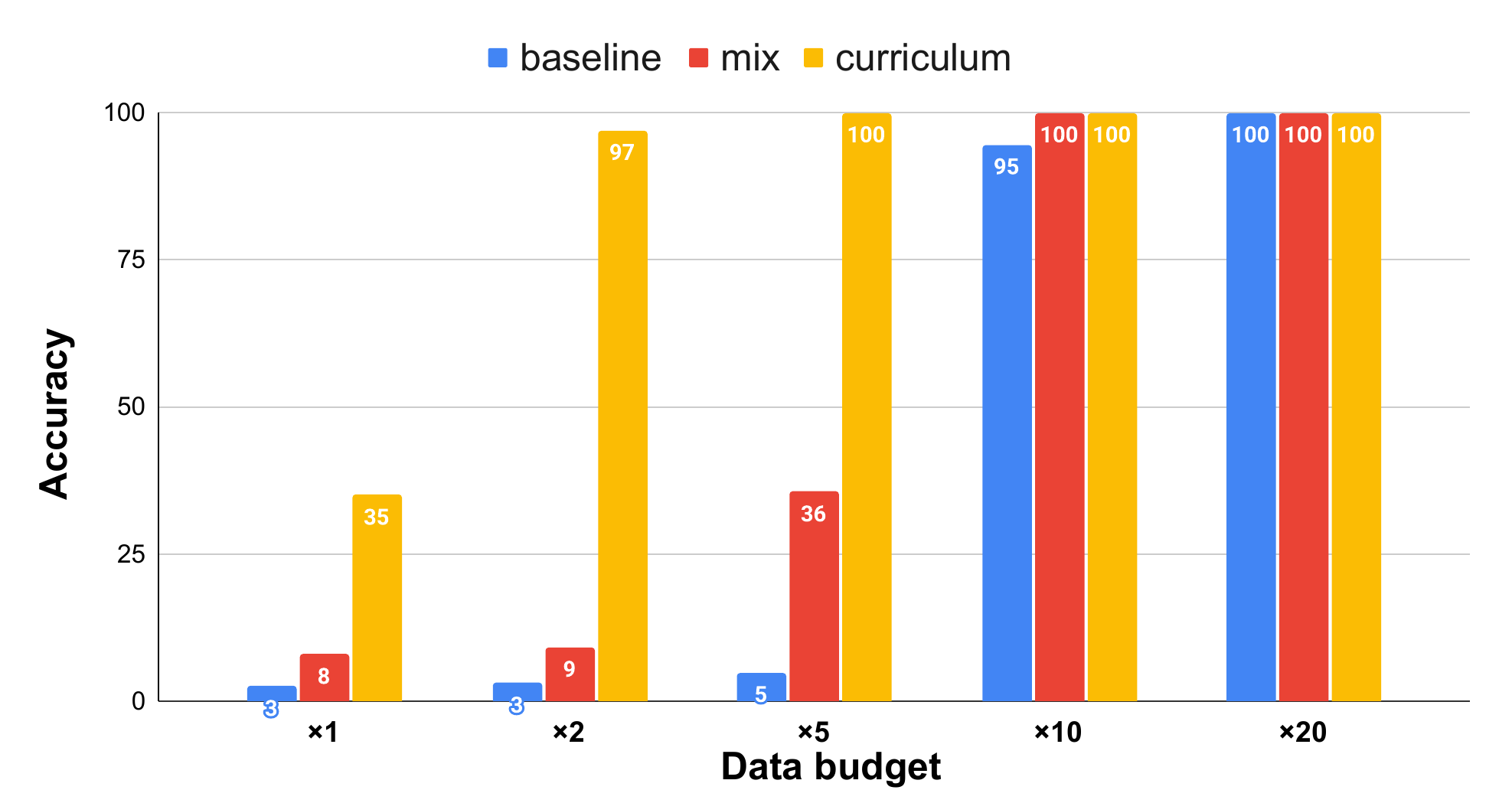}
    \caption{\threehopl}
    \label{fig:hop3_large_mix_cl}
  \end{subfigure}
  \hfill
  \begin{subfigure}[b]{0.46\textwidth}
    \includegraphics[width=\linewidth]{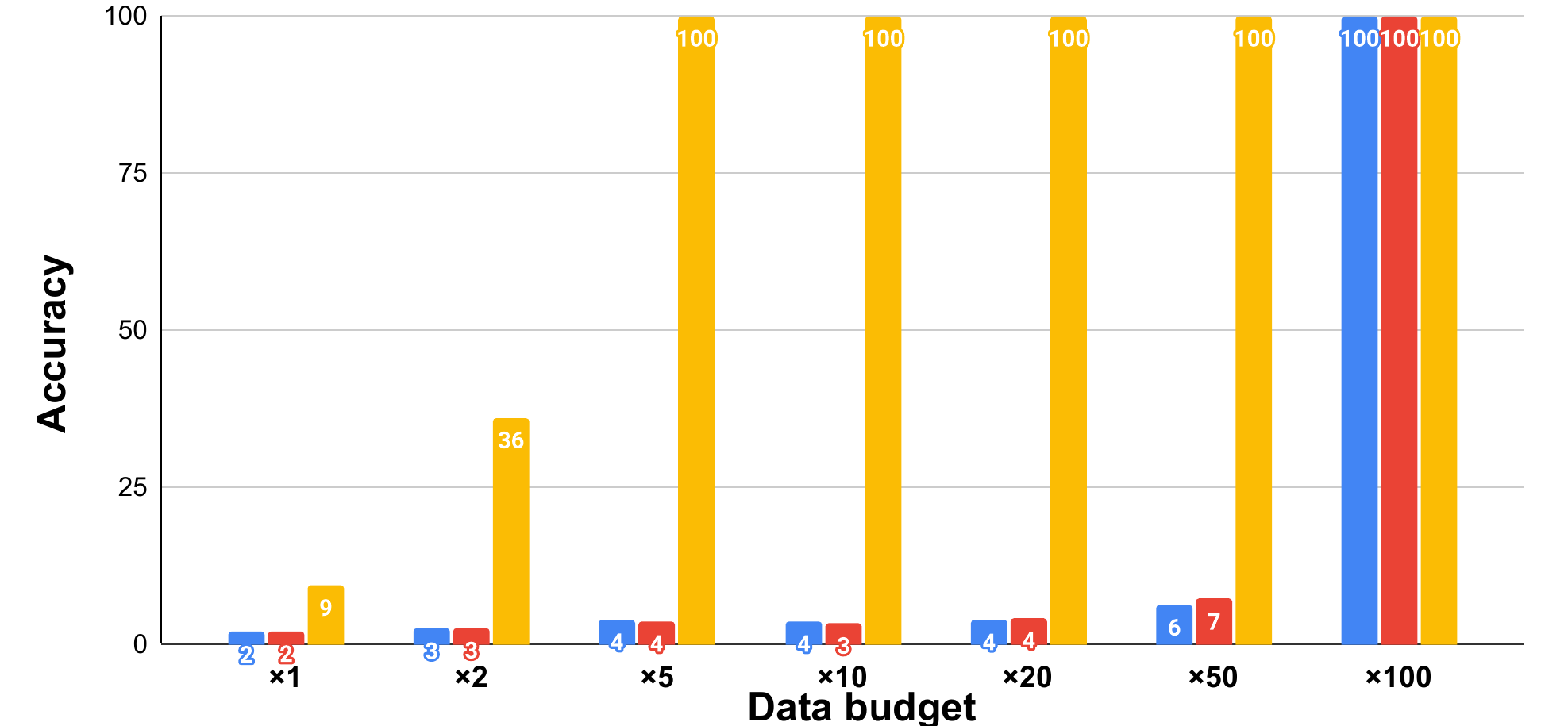}
    \caption{\fourhopl}
    \label{fig:hop4_large_mix_cl}
  \end{subfigure}
  \caption{Model performance on \hopl\ datasets with mixed learning and curriculum learning.
  Curriculum learning enables the model to solve the $4$-hop task with a $\times5$ training budget, compared to the $\times100$ required by both the baseline and mixed learning setups.}
  \label{fig:large_mix_cl}
\end{figure}

Finally, we study training strategies to improve the data budget issue.  
Models in Section \ref{sec:e1} were trained solely on $k$-hop task, but  $i$-hop ($i < k$) questions should also be available in realistic setups.
By exploiting such easier questions as additional training data for $k$-hop task, we demonstrate that curriculum learning significantly mitigates the exponential growth issue, though not eliminate the increase of data budget as $k$ increases.

\subsection{Experiment setup}
We use the same GPT-2 architecture as in Section~\ref{sec:exp_setup} and compare two strategies: mixed learning and curriculum learning \citep{bengio2009curriculum}.

\paragraph{Mixed learning.} 
We construct the training set by combining reasoning questions from both lower-hop and $k$-hop tasks. 
For instance, the \fourhop\ training set contains a mix of \twohop, \threehop, and \fourhop\ questions, along with all relevant entity profiles. 
Lower-hop questions are generated using the same entity profiles as the target task. 
We vary the $k$-hop training budget using the same scaling factors as in Section~\ref{sec:exp_setup}, while keeping the amount of lower-hop data fixed (see Appendix~\ref{app:dataset:cl} for dataset details).

\paragraph{Curriculum learning.} 
Training in curriculum learning is split into multiple stages, where each stage progressively introduces harder reasoning tasks. 
For a $k$-hop task, training proceeds in $k{-}1$ stages: the first stage uses only \twohop\ questions, the second stage includes both \twohop\ and \threehop, and so on.
We use the same lower-hop data as in the mixed learning setup and ensure that total training steps are equal across both strategies. 
See Appendix \ref{app:cl_training} for training details.

\paragraph{Test set.}
To avoid shortcut solutions (e.g., where lower-hop queries appear as subcomponents of the $k$-hop query), we generate test sets such that such overlaps do not exist using rejection sampling. 
The test set size remains 3000 instances, consistent with previous experiments.

\subsection{Results.}
Figures ~\ref{fig:large_mix_cl} shows the results for \hopl, and the same pattern holds for \hops\ (see Appendix \ref{app:detailed_results:hops} for results).
We compare mixed learning, curriculum learning, and the baseline model trained only on the target $k$-hop dataset from Section~\ref{sec:exp_setup}.

\paragraph{Curriculum learning significantly reduces the required data budget.}
Notably, curriculum learning yields the most significant improvement.
For example, perfect accuracy on \fourhop\ tasks is achieved with only a $\times5$ budget, compared to $\times100$ in the baseline.
In contrast, simply mixing all available data provides only modest gains.
This demonstrates that presenting easier reasoning tasks before harder ones is a highly effective strategy for improving data efficiency.

\paragraph{Curriculum learning builds circuits gradually.}
We attribute this effectiveness of curriculum learning to a stepwise build-up of circuits: As we show in Appendix \ref{app:mi:cl}, mechanisms retrieving lower-hop entities (e.g., $1$-hop) emerge in the early training stages; subsequent stages then build upon these established circuits to learn more complex reasoning tasks.
While baseline models have to construct a full circuit for $k$-hop reasoning at once, curriculum learning enables $1$-hop circuits to emerge in shallower layers in the first stage, with later stages developing circuits for $2$-hop and $3$-hop entities on top of these.

\paragraph{Curriculum learning does not completely solve the data growth issue.}
Despite the effectiveness of curriculum learning strategy, it does not completely eliminate the growth of data budget.
For example, curriculum learning requires $\times2$ budget for \threehop\ task and $\times5$ for \fourhop\ task, indicating the challenge of \hop\ implicit reasoning for LMs.

\section{Conclusion} \label{sec:conclusion}

Our work investigates whether language models can learn implicit multi-hop reasoning.
We provide a nuanced answer through controlled $k$-hop reasoning datasets using GPT2-style language models.
On the one hand, our findings demonstrate that language models can indeed learn \hop\ reasoning through sequential lookup of intermediate bridge entities layer by layer. 
However, this capability comes at a cost: as $k$ increases, the training data budget grows exponentially, and the model depth must scale linearly.
Furthermore, while curriculum learning mitigates the data budget growth, it does not eliminate the growth trend. 
Together, we present a comprehensive view of the potential and limitations of LMs in implicit reasoning, underscoring the inherent trade-offs between task complexity, data requirements, and model depth.

\section{Limitations} \label{sec:limitation}
We limit our study to implicit reasoning tasks using synthetic datasets generated based on predefined templates.
Applying the same analysis to realistic datasets is challenging due to the difficulty of collecting complex multi-hop questions (e.g.\ \fourhop\ questions) and corresponding facts.
Due to computational budget constraints, we also restrict our experiments to \hop\ tasks with $k < 5$.

Additionally, our experiments primarily rely on randomly initialized small language models (GPT-2 small).
While we also observe that the data budget issues persist for pretrained models (e.g.\ pretrained GPT-2) and larger models (GPT-2 medium and large with up to 770M parameters), we do not extend our analysis to models with greater parameter sizes.

\section*{Acknowledgments}
We thank Lijie Chen for fruitful discussions, and Ji-Ung Lee and Katharina Stein for their feedback on the paper.



\bibliography{anthology,custom}
\appendix



\section{Details for Theoretical Results}\label{app:theory}

\subsection{Proof of Theoretical Bound}

\begin{theorem}[Restated from \ref{thm:main-theorem}]
    Consider a causal transformer operating in $p$ bits of precision, with $d$ hidden units, $H$ heads and $L$ layers. Assume it performs $k$-hop reasoning over $\mathcal{E}$ and $\mathcal{R}$ as defined above.
    Assume further that the attention pattern does not depend on $e$.
    If $k \leq |\mathcal{E}|-2$, then, for some  $\mathcal{R}$,
    \begin{equation}
    L \geq \frac{k}{8pdH} 
\end{equation}
\end{theorem}

\begin{proof}
Recall that the input has the form
\begin{center}
\textit{Who is the $r_k$  of the  $r_{k-1}$ \dots $r_{2}$ of the  $r_1$  of   $e$ ? Answer : }
\end{center}
Our argument is based on communication complexity.
We consider a communication game where Alice holds $r_1, \dots, r_k$, and Bob holds $e$.
Due to causal masking, Alice can compute the transformer's activations on all tokens ``Who is the $r_k$ of the \dots $r_1$ of'' without receiving any information from Bob.
In order to compute activations on then final tokens ``$e$ ? Answer :'' (and thus the prediction), Bob requires access to the outputs of attention heads on these tokens. Because the attention patterns are assumed to be independent of the query $e$, Alice can simple, for each head at the four tokens ``$e$ ? Answer :'' provide the activations within the span known to Alice weighted with the attention weights.
Thus, a total of $4HL$ such activation vectors is sufficient. Furthermore, each of these activations can be encoded with $pd$ bits. Overall, thus, Bob can compute the output with access to only $4LpdH$ bits.

That is, there a way to compress the composed function $r_k \circ  \dots \circ r_1$ into $4LpdH$ bits.
Hence, $2^{4LpdH}$ upper-bounds the cardinality of such possible functions:
\begin{equation}\label{eq:comm-lower-bound}
    4LpdH \geq \log_2 \left| \{ r_k \circ  \dots \circ r_1 : r_1, \dots, r_k \in \mathcal{R} \}\right|
\end{equation}
We note that, in general, the right-hand-side could be small: for instance, if $\mathcal{R}$ just contains the identity function, then the set of $k$-fold composed functions will also just contain the identity function (since its composition with itself again equals itself). To conclude the theorem, the remaining problem is now to show that there is a way of choosing $\mathcal{R}$ for which the right-hand-side scales with $k$.

We arbitrarily label the elements of $\mathcal{E}$ as $\{0, x_0, x_1, \dots, x_{n-1}\}$, and define $\mathcal{R} = \{f,g\}$ where:
\begin{align*}
f(0) =& 0 \\
    f(x_i) =& x_{(i+1) \% n} \\
    g(0) =& 0 \\
    g(x_0) =& 0 \\
    g(x_i) = & x_i\ \  (i>0)
\end{align*}
Intuitively, (1) there is a special ``sink'' entity 0, (2) $f$ cyclically shuffles the non-sink entities, (3) $g$ maps the first entity in the order to the sink entity.
We now consider all words over $f,g$ of length $k$ where $gg$ does not occur (recall that, by assumption, $k \leq |\mathcal{E}|-2 = n-1$).
The number of such words is exponential in $k$; indeed, it is at least $\left(\frac{3}{2}\right)^{k}$.\footnote{Indeed, it equals the Fibonacci number $F_{k+2}$. For large $k$, this is lower-bounded by $\left(\frac{3}{2}\right)^{k+2}$; to make it valid even for small $k$, it is sufficient to instead lower-bound by $\left(\frac{3}{2}\right)^{k}$.} 
Each such word, interpreted as a composition, generates a different transformation $\mathcal{E} \rightarrow \mathcal{E}$.\footnote{We note the connection to the general fact that every finite semigroup can be embedded into a finite semigroup generated by an idempotent and a nilpotent (note that $f$ is nilpotent and $g$ is idempotent), proven using a similar construction in Theorem 1.1 of \citet{higgins2017embedding}.}
Indeed, to show this, we simply note that such a composition maps $x_i$ to $0$ if and only if $g$ was applied immediately after the $i+1$-th application of $f$.
Thus, when $k \leq |\mathcal{E}|-2$, we have lower-bounded the right-hand-side of (\ref{eq:comm-lower-bound}) as $ k \cdot \log_2 \frac{3}{2} > \frac{k}{2}$.
The theorem then follows by rearranging (\ref{eq:comm-lower-bound}).    
\end{proof}

\subsection{Discussion}\label{app:theory-discussion}

\paragraph{Related Work}
\citet{chen2024theoretical} prove a lower bound on $L$ for causal transformers solving a more complicated kind of  composition task, which composes functions taking \emph{two} arguments. The input provides both (i) a sequence of functions, and (ii) a sequence of entities serving as the second argument, with the output
\begin{equation}
    z_{l+1}(w_l, z_l(w_{l-1}, z_{l-1}( \dots z_2(w_1,i_1))))
\end{equation}
where $z_1, \dots, z_{l_1}$ are two-argument functions, $i_1$ can be viewed as an input entity (similar to the query in our $k$-hop task), and---crucially---$w_1, \dots, w_l$ serve as additional arguments to the two-argument functions.
For this more complicated task, \citet{chen2024theoretical} prove a depth-width tradeoff.
Unlike our result, theirs does not make any assumption on the attention patterns, however, it is specifically proven for this more complicated task.
Intuitively, separately presenting the functions $z_1, \dots, z_{l_1}$ from the entities $w_1, \dots, w_l$ serving as their second argument might play a key role in making the task challenging enough to enable the theoretical analysis in \citet{chen2024theoretical}.
Hence, it appears to remains open if such a bound can also be proven for a task directly matching $k$-hop reasoning (Figure~\ref{fig:dataset_example}). 

Another line of work has shown limitations of \emph{one-layer} transformers in performing function composition \citep{peng2024limitations, kozachinskiy2025strassen}; this is consistent with our evidence that $k$-hop tasks require increasing numbers of layers, but does not bound how many layers are needed.

\paragraph{Bounds from $NC^1$-hardness}
As transformers can be simulated in $TC^0$ \citep[e.g.][]{merrill-sabharwal-2023-parallelism, strobl2023average}, some work has obtained bounds conditional on standard complexity conjecture $TC^0 \neq NC^1$. Assuming this conjecture, transformers generally cannot solve $k$-hop composition unless the number of layers increases as $k$ increases; as an example, consider $\mathcal{E}$ to be $\{1,\dots, 5\}$, and $\mathcal{R}$ a generating subset of the alternating permutation group $A_5$; then solving $k$-hop composition is $NC^1$-hard and predicted to not be feasible for transformers.
However, due to the difficulty of proving lower bounds for $TC^0$, this line of reasoning does not provide precise information about how quickly exactly $L$ needs to grow with $k$.

\paragraph{Role of attention pattern}
Theorem~\ref{thm:main-theorem} applies in the case where attention patterns do not depend on the input entity $e$ (in fact, they might still depend on $r_1, \dots, r_k$). Our proof strategy makes use of this assumption, because it limits the amount of information that any individual attention head at the final positions can obtain about $r_1, \dots, r_k$. It remains open if this assumption can be relaxed. Intuitively, it does not seem clear how changing attention patterns could make the task easier. However, formally proving lower bounds for multi-layer transformer without either constraining attention patterns (as we do) or considering a more complex task (as done in \citet{chen2024theoretical}) remains challenging; we expect that further technical tools will be needed to overcome these challenges.

\section{Effect of more powerful models} \label{app:scaleup}
Section \ref{sec:e1} only presents results for our randomly initialized GPT-2 small model.
Would the training data budget still grows in an exponentially way as the $k$ increases, even with more powerful language models?
We investigate this question by applying the same experiments in Section \ref{sec:base_results} with two other model setups: finetuning and scaling up model parameters.
We only report 1-run results for all experiments in this section.

\subsection{Finetuning}
For finetuning setup, we finetune a pretrained language model on the same training set in Section \ref{sec:e1} and evaluate it on the same test set. 
Here we start with the pretrained GPT-2 small model \citep{radford2019language}, and use the same hyperparameters as for our randomly initialized GPT-2.
Note that the pretrained GPT-2 adopts a learned positional embedding \citep{vaswani2017attention} instead of RoPE \citep{su2024roformer}, and thus we cannot directly tell the effect of pretraining compared to non-pretrained model.
Here we only use this experiment to confirm that the significant increase of data budget still holds for pretrained models.

\begin{table*}[t]
  \centering
  \small
  \setlength{\tabcolsep}{4pt}
  \begin{tabular}{@{}ll|ccccccc@{}}
    \toprule
    & & \multicolumn{7}{c}{Training data budget} \\
    \multicolumn{2}{c}{Dataset}  & ×1 & ×2 & ×5 & ×10 & ×20 & ×50 & ×100 \\
    \midrule
    \multirow{3}{*}{\hops}
      & 2-hop & \asbar{93.7} & \dasbar{100} & \dasbar{100} & \dasbar{100} & \dasbar{100} & \dasbar{100} & \dasbar{100} \\
      & 3-hop & \asbar{6.6}  & \asbar{8.2} & \asbar{45.4} & \asbar{100} & \dasbar{100} & \dasbar{100} & \dasbar{100} \\
      & 4-hop & \asbar{5.1}  & \asbar{6.5} & \asbar{6.9} & \asbar{8.4} & \asbar{23.8} & \asbar{100} & \asbar{100} \\
    \midrule
    \multirow{3}{*}{\hopl}
      & 2-hop & \asbar{100} & \dasbar{100} & \dasbar{100} & \dasbar{100} & \dasbar{100} & \dasbar{100} & \dasbar{100} \\
      & 3-hop & \asbar{2.8}  & \asbar{2.5} & \asbar{3.9} & \asbar{87.5} & \asbar{100} & \dasbar{100} & \dasbar{100} \\
      & 4-hop & \asbar{2.3}  & \asbar{3}   & \asbar{4.1} & \asbar{4.2} & \asbar{3.9} & \asbar{24.6} & \asbar{100} \\
    \bottomrule
  \end{tabular}
  \caption{Accuracy of finetuned GPT-2 small models on \hops\ and \hopl\ datasets with different training data budgets.}
  \label{tab:gpt2_pretrained}
\end{table*}

Table \ref{tab:gpt2_pretrained} presents the results of pretrained GPT-2 small models.
Overall, the data budget still exponentially grows as the $k$ value increases. 
On \hopl\ the model needs $\times10$ budget for \threehop task and $\times100$ for \fourhop, which is the same as our randomly initialized transformer.
The pretrained model achieves lower accuracy on \hops\ datasets, e.g.\ only $93.7\%$ accuracy on \twohop\ task.
Nonetheless, the model still learns to perfectly solve \hops datasets with enough data budget, e.g.\, for \threehop, model accuracy gets significantly improved accuracy at $\times5$ budget and reaches $100\%$ at $\times10$ budget.
We consider the lower accuracy here is likely due to the lack of hyperparameter optimization and use of better position encoding.  

\subsection{Scaling up the model size.}

\begin{table*}[t]
  \centering
  \small
  \setlength{\tabcolsep}{4pt}
  \begin{tabular}{@{}l l | c c c c c c c@{}}
    \toprule
    Model            & Dataset    & \multicolumn{7}{c}{Training data budget} \\
    \cmidrule(lr){3-9}
                     &        & ×1   & ×2   & ×5    & ×10  & ×20   & ×50   & ×100 \\
    \midrule
    \multirow{3}{*}{GPT-2 Medium}
      & \twohops  & \asbar{100} & \dasbar{100} & \dasbar{100} & \dasbar{100} & \dasbar{100} & \dasbar{100} & \dasbar{100} \\
      & \threehops  &  \asbar{5.6} &  \asbar{13.9} &  \asbar{99.9} &  \asbar{100} &  \dasbar{100} &  \dasbar{100} &  \dasbar{100} \\
      & \fourhops  &  \asbar{5}   &   \asbar{5.8} &   \asbar{8}   &  \asbar{10.2}&  \asbar{99.8}&  \asbar{100} &  \asbar{100} \\
    \midrule
    \multirow{3}{*}{GPT-2 Large}
      & \twohops  &   \asbar{100}          &     \dasbar{100}        &  \dasbar{100}           &    \dasbar{100}         &      \dasbar{100}       &     \dasbar{100}        &      \dasbar{100}       \\
      & \threehops  &       \asbar{6.5}      &       \asbar{22.0}      &        \asbar{100}     &     \asbar{100}        &   \dasbar{100}          &   \dasbar{100}          &    \dasbar{100}         \\
      & \fourhops  &      \asbar{4.6}       &    \asbar{5.6}         &        \asbar{7.9}     &      \asbar{11.0}       &    \asbar{99.6}         &    \asbar{100}         &     \asbar{100}        \\
    \bottomrule
  \end{tabular}
  \caption{Accuracy of GPT-2 Medium and Large on \hops\ datasets with different training data budgets. Empty cells indicate that the data budget exceeds the number of available questions possible to generate.}
  \label{tab:gpt2_comparison}
\end{table*}

We also evaluate setups where we scale up the number of model parameters.
\citet{kaplan2020scalinglawsneurallanguage} demonstrates that larger model size is crucial to gain high performance, especially the depth of transformer layers \citep{fagnou-etal-2024-chain,ye2024physics},  and we want to investigate if larger models address the data budget issue.
Here we experiment with same architecture described in Section \ref{sec:exp_setup} (i.e.\ GPT-2 with RoPE), and we set hyperparameters of architectures (e.g.\ number of layers, attention heads, etc.) according to the GPT-2 medium 
\footnote{https://huggingface.co/openai-community/gpt2-medium}
and large model
\footnote{https://huggingface.co/openai-community/gpt2-large}
configurations.
We randomly initialize the model and train and evaluate it on the \hops\ datasets in Section \ref{sec:exp_setup} from scratch.
Hyperparameters for training are the same as Section \ref{sec:exp_setup}.

Table \ref{tab:gpt2_comparison} reports the results of such larger models. 
For both GPT-2 medium and large sized models, the growth of data budget issue still persists. 

\section{Effect of model depth} \label{app:depth_effect}
The analysis in Section \ref{sec:theoretical_analysis} theoretically proves that the depth of the model needs to grow linearly as the value $k$ increases.
In this section, we further provide empirical evidence to support our theory by showing that a transformer with fewer layers struggles with $k$-hop reasoning.

We conducted preliminary experiments in which we vary the model depth from 2 to 5 layers and use the \hops\ datasets. 
We adopt the training data budgets that are sufficient for a 12-layer GPT-2-small model to succeed (see Table \ref{tab:baseline}): 
$\times1$ for 2-hop, $\times5$ for 3-hop, and $\times20$ for 4-hop. 
All other hyperparameters and training configurations are unchanged, except for the model depth (number of transformer layers).

\begin{table}[h]
\centering
\footnotesize 
\begin{tabular}{lcccc}
\toprule
Dataset & 2 layers & 3 layers & 4 layers & 5 layers \\
\midrule
\footnotesize \twohops & 31 & \textbf{96} & \textbf{97} & \textbf{100} \\
\footnotesize  \threehops & 13 & 55 & \textbf{100} & \textbf{100} \\
\footnotesize  \fourhops & 25 & 38 & 83 & \textbf{95} \\
\bottomrule
\end{tabular}
\caption{Test accuracy with varying model depth. Accuracies above 90\% are boldfaced.}
\label{tab:depth_effect}
\end{table}

We report test accuracy in Table \ref{tab:depth_effect}. 
Noticeably, the required model depth grows as the hop number increases, which is consistent with what our theoretical analysis predicts.

\section{Dataset details} \label{app:dataset}

\subsection{Datasets in Section \ref{sec:e1}} \label{app:dataset:baseline}

\begin{table*}[!htp]
  \centering
  \scriptsize
  \setlength{\tabcolsep}{4pt}
  \begin{tabularx}{\textwidth}{@{}llXXXXXXX@{}}
    \toprule
    Dataset  & Hop   & ×1     & ×2     & ×5      & ×10     & ×20      & ×50       & ×100      \\
    \midrule
    \multirow{3}{*}{small}
      & 2-hop & 4 250  &   &    &    &     &      &      \\
      & 3-hop & 4 250  & 8 250  & 20 250  & 40 250  &     &      &      \\
      & 4-hop & 4 250  & 8 250  & 20 250  & 40 250  & 80 250    & 200 250      & 400 250      \\
    \addlinespace
    \multirow{3}{*}{large}
      & 2-hop & 20 500 &   &    &    &     &      &      \\
      & 3-hop & 20 500 & 40 500 & 100 500 & 200 500 & 400 500   &      &      \\
      & 4-hop & 20 500 & 40 500 & 100 500 & 200 500 & 400 500   & 1 000 500    & 2 000 500    \\
    \bottomrule
  \end{tabularx}
  \caption{Statistics of the number of training instances in each setup.}
  \label{tab:dataset_statistics}
\end{table*}

\begin{table}[!ht]
  \centering
  \scriptsize
  \setlength{\tabcolsep}{6pt}
  \begin{tabularx}{0.8\linewidth}{@{}*{4}{X}@{}}
    \toprule
    instructor   & teacher       & ruler        & advisor       \\
    supervisor   & leader        & manager      & director      \\
    patron       & mentor        & administrator& coordinator   \\
    tutor        & predecessor   & sponsor      & financier     \\
    backer       & overseer      & employer     & boss          \\
    \bottomrule
  \end{tabularx}
  \caption{Vocabulary of relation names.}
  \label{tab:relation_vocab}
\end{table}

\begin{table}[!ht]
  \centering
  \scriptsize
  \setlength{\tabcolsep}{6pt}
  \begin{tabularx}{0.8\linewidth}{@{}*{4}{X}@{}}
    \toprule
    Emil   & Gavin       & Chad        & Flora       \\
    Adam   & Addie        & Bobby      & Edwin      \\
    Gabby       & Helen        & Jeffery & Joel   \\
    Kris        & Kristen   & Lisa      & Liam     \\
    Eva       & Emma      & Dylan     & Isabella          \\
    \bottomrule
  \end{tabularx}
  \caption{Subset of vocabulary of entity names.}
  \label{tab:entity_vocab}
\end{table}

\paragraph{Namespaces of entity and relation names.}
We provide details on the entity and relation namespaces used to generate the datasets in Section \ref{sec:dataset}. 
Our dataset consists of \mathe\ entities, each with a distinct name and $N$ relations. 
We use 600 unique single-token person names (e.g., \textit{Jennifer}) and 20 single-token relation names (e.g., \textit{instructor}), generated by ChatGPT\footnote{https://chatgpt.com/}, as the namespaces for entities and relations, respectively.
The complete vocabulary of relation and a subset of entity names are provided in Tables \ref{tab:relation_vocab} and \ref{tab:entity_vocab}.
Since our main experiments use a randomly initialized language model, the specific choice of vocabulary does not influence our conclusions.

\paragraph{Top-hierarchy entity profiles.}
Entities in the top layer of Figure \ref{fig:datageneration} are not linked to any targets, making it non-trivial to generate their profiles. 
Nevertheless, we include their profiles in the training set to maintain consistency across all $k$-hop tasks with $k \in \{2, 3, 4\}$. 
In both the \twohop\ and \threehop\ tasks, answer tokens (i.e., entity names) appear in the training set as subject entities in their own profiles. 
To ensure the same holds for the \fourhop\ task, where answers correspond to top-layer entities, we generate profiles for these entities as well.
Specifically, we generate these profiles by concatenating facts in which the subject entity is the top-layer entity itself, the relation is one from Table \ref{tab:relation_vocab}, and the object entity is a single-token name sampled from an additional set of 100 person names. 
These object names are distinct from the ones used in Figure \ref{fig:datageneration}. 
Since these facts are never used in any $k$-hop question in the training or test sets, including them does not affect our results or conclusions.
  

Table \ref{tab:dataset_statistics} reports the training set sizes for each dataset configuration. 
To maintain consistency across data budget setups, we include the same set of \mathe\ entity profiles (e.g., \mathe$=250$ profiles for \hops) in each training set. 
We partition the \mathe\ entities into 5 disjoint subsets, each containing \mathe$/5$ entities, and only generate reasoning questions targeting one subset (e.g., entities in the bottom hierarchy of Figure \ref{fig:datageneration}).
Each entity profile includes \mathr\ relations (e.g., \mathr$=10$ for \hops), allowing us to generate \mathr$^2 \times$ \mathe$/5 = 5000$ questions for \twohops, of which $80\%$ are randomly selected as training instances.

\subsection{Datasets for mixed and curriculum learning} \label{app:dataset:cl}
In mixed learning, we introduce lower-hop reasoning questions as auxiliary training instances to facilitate learning more complex reasoning tasks. 
For the $3$-hop task, we add $2$-hop instances, and for the $4$-hop task, we add both $2$-hop and $3$-hop instances. 
For \hops, we include 4k $2$-hop instances as auxiliary data for \threehops, and 4k $2$-hop and 20k $3$-hop instances for \fourhops. 
For \hopl, we include 32k $2$-hop instances for \threehopl, and 32k $2$-hop and 100k $3$-hop instances for \fourhopl.
Due to computational constraints, we did not specifically tune the size of auxiliary data. 
The curriculum learning setup uses the same auxiliary instances as mixed learning.

\section{Training details} \label{app:training}
\subsection{Baseline} \label{app:baseline_training}
This section provides the model architecture and training setup used in Section \ref{sec:e1}. 
Unless stated otherwise, the same configuration is applied across all experiments in this paper.   

\begin{figure*}[!t]
  \centering
  \begin{subfigure}[b]{0.48\textwidth}
    \includegraphics[width=\linewidth]{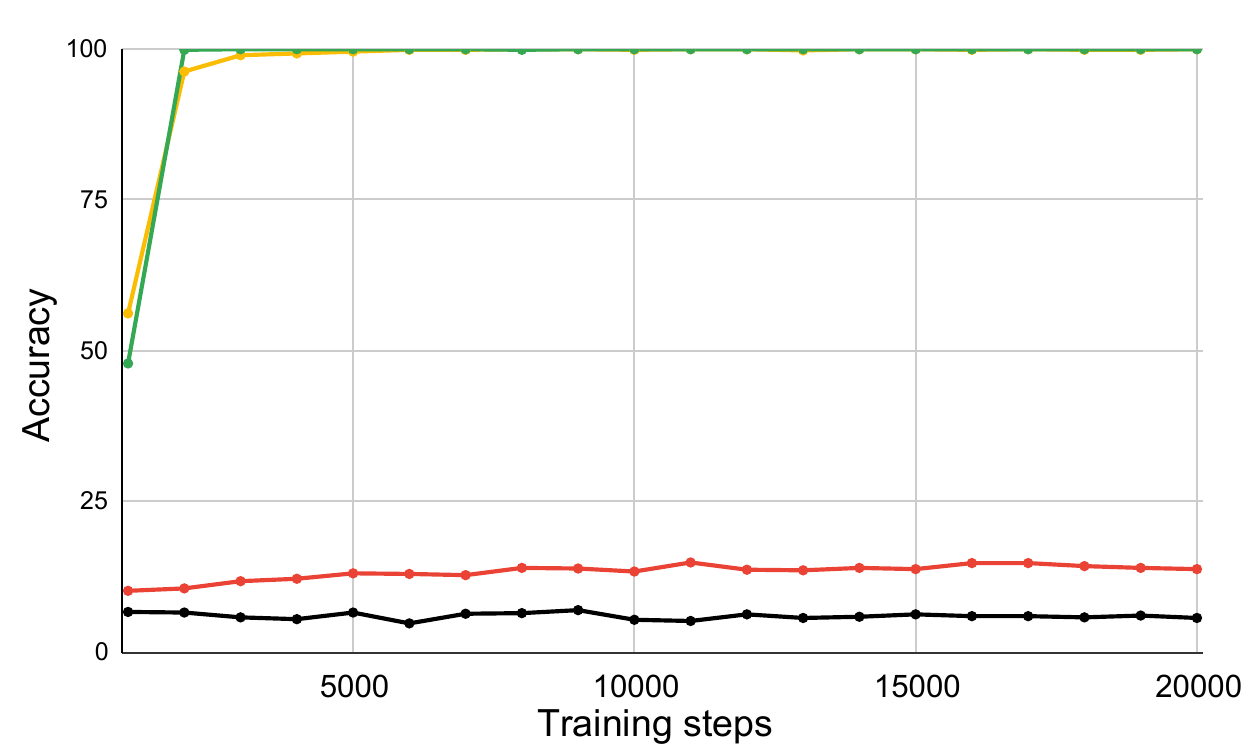}
    \caption{\threehops}
    \label{fig:hop3_small}
  \end{subfigure}
  \hfill
  \begin{subfigure}[b]{0.48\textwidth}
    \includegraphics[width=\linewidth]{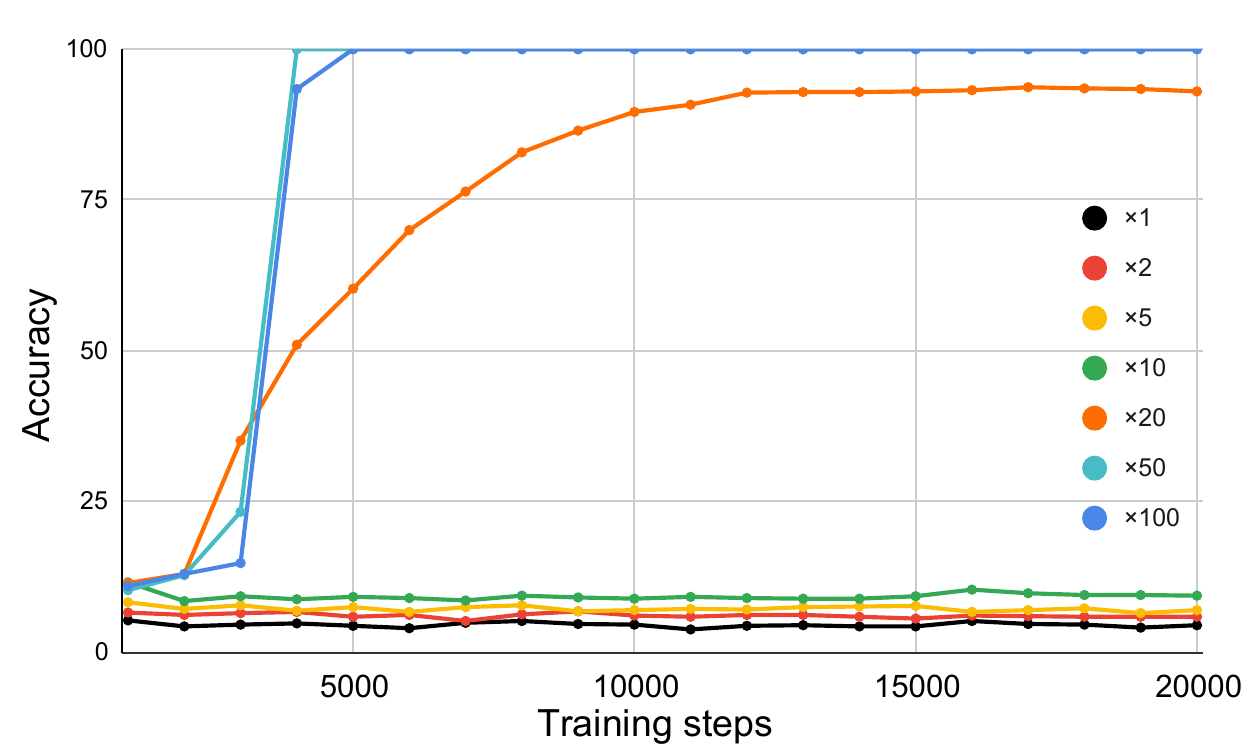}
    \caption{\fourhops}
    \label{fig:hop4_small}
  \end{subfigure}
  \caption{Model accuracy on \threehops\ and \fourhops. x-axis refers to the number optimization steps. }
  \label{fig:small_model_comparison}
\end{figure*}

\begin{figure}[!t]
  \centering
  \begin{subfigure}[b]{0.48\textwidth}
    \includegraphics[width=\linewidth]{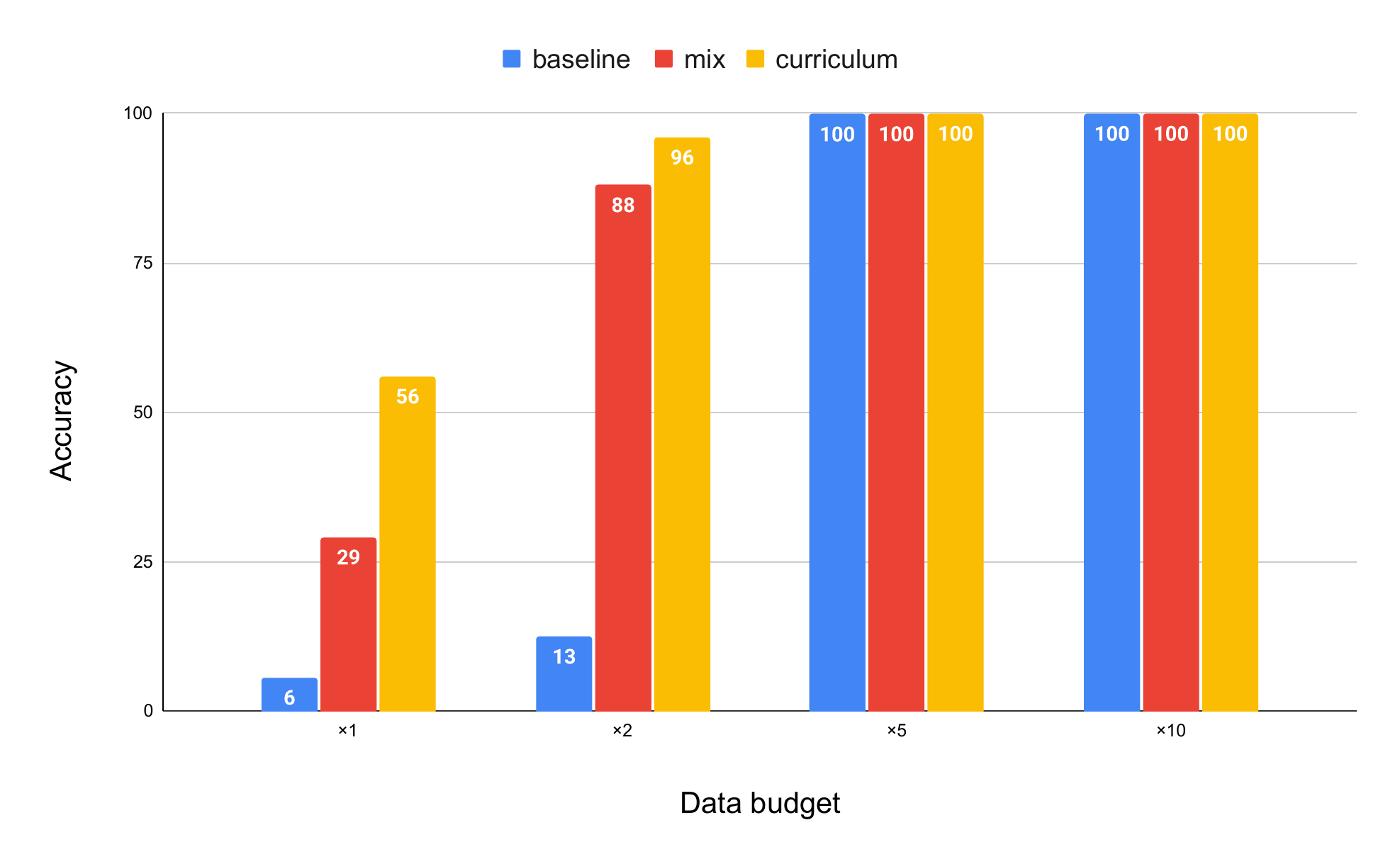}
    \caption{\threehops}
    \label{fig:hop3_small_mix_cl}
  \end{subfigure}
  \hfill
  \begin{subfigure}[b]{0.48\textwidth}
    \includegraphics[width=\linewidth]{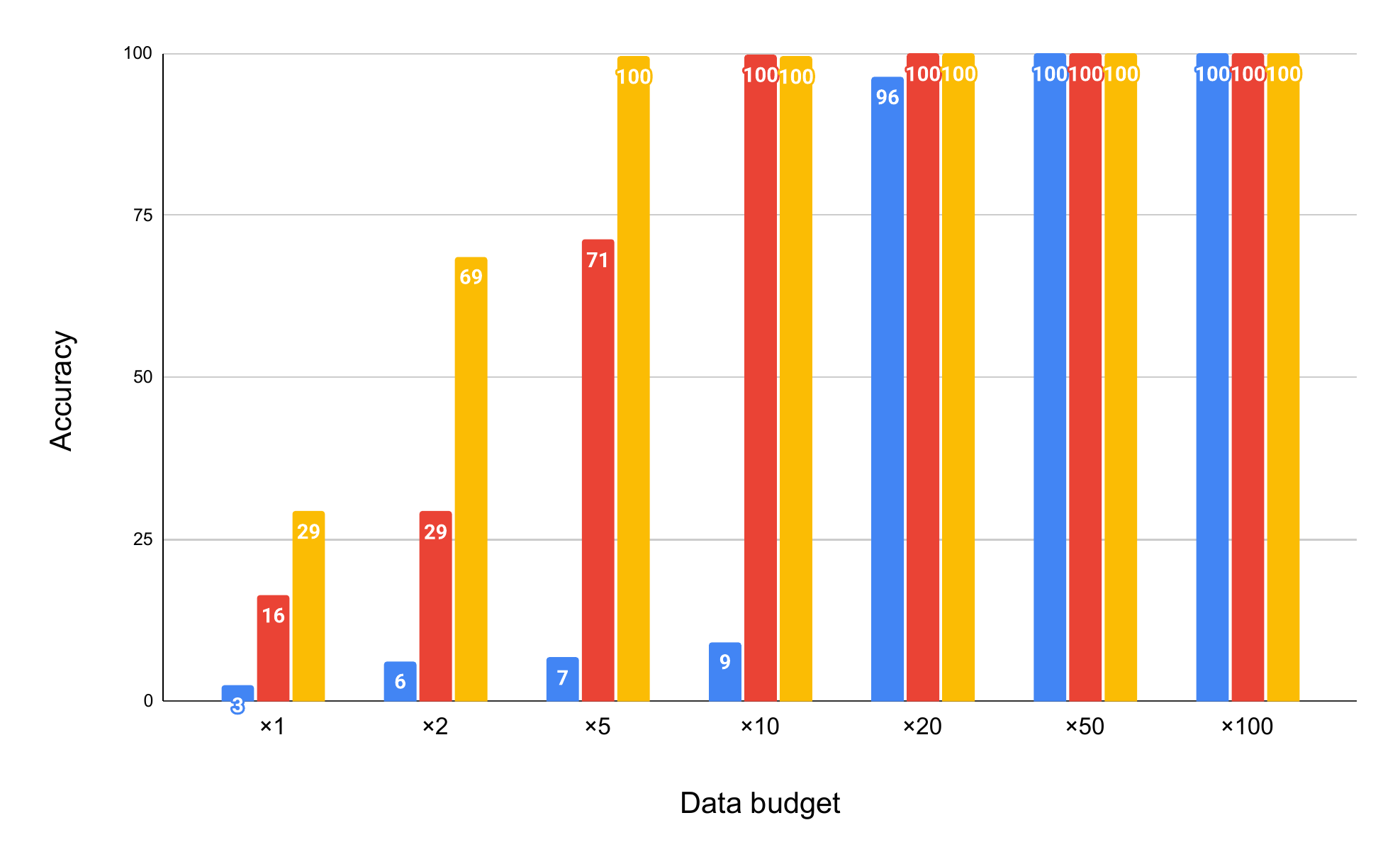}
    \caption{\fourhops}
    \label{fig:hop4_small_mix_cl}
  \end{subfigure}
  \caption{Model performance on \hops\ datasets with mixed learning and curriculum learning.}
  \label{fig:small_mix_cl}
\end{figure}

\paragraph{Model architecture.}
We adopt the GPT-2 small architecture\footnote{https://huggingface.co/openai-community/gpt2}, consisting of 12 transformer layers with 12 attention heads. 
The input embedding dimension is 768, and the context window is limited to 1024 tokens. 
Instead of absolute position embeddings used in the original Transformer \citep{vaswani2017attention}, we employ Rotary Position Embedding (RoPE) \citep{su2024roformer} to encode positional information. 
We use the default GPT-2 tokenizer and extend the vocabulary to include all entity profile names (e.g., \textit{Jennifer}), resulting in a vocabulary size of $|V| = 50,740$.

\paragraph{Training.}
The batch size is set to 512 with gradient accumulation steps of 4. 
We use the AdamW optimizer \citep{DBLP:journals/corr/KingmaB14, loshchilov2019decoupledweightdecayregularization} with the following hyperparameters: learning rate of $5e-4$, $\epsilon = 1e-6$, $\beta_1 = 0.9$, $\beta_2 = 0.999$, and weight decay of $0.1$.  
Training begins with a 1k-step warm-up phase, followed by a cosine learning rate scheduler \citep{loshchilov2016sgdr}, with a minimum learning rate set to $0.1 \times$ the initial learning rate.

Experiments are run on Nvidia A100 and H100 GPU cards (80GB). 
Each experiment is conducted on one single GPU, which takes about 8 hours for 20k optimization steps.
The implementation is based on Huggingface \citep{wolf2019huggingface} and Pytorch \citep{paszke2019pytorch}.
GPT-2 is released under the MIT License by OpenAI.

\subsection{Setup for mixed and curriculum learning} \label{app:cl_training}
The model architecture for mixed and curriculum learning experiments remains the same as the baseline configuration described in Section \ref{app:baseline_training}. 
The training setup for mixed learning also follows the baseline training setup without any modifications.  

Training in curriculum learning is divided into multiple stages, where each stage progressively introduces harder reasoning tasks. 
For a $k$-hop task, training consists of $k{-}1$ stages:  
The first stage includes only \twohop\ questions. 
The second stage adds \threehop\ questions. 
The third stage adds \fourhop\ questions to the training set (only applicable for \fourhop\ tasks).
Hence we have 2 training stages for \threehop\ task and 3 training stages for \fourhop\ task.
The maximum number of training steps for each stage across different target tasks is reported in Table \ref{tab: cl_stage}. 
Each stage employs the same learning rate scheduler and warm-up steps as in the baseline training setup to maintain consistency.
The batch size and gradient accumulation steps remain the same as in the baseline setup. 

\begin{table}[!htp]\centering
\footnotesize
\begin{tabular}{lrrrrr}\toprule
Task &Stage 1 &Stage 2 &Stage 3 &Total \\
\midrule
\threehops &10000 &10000 &- &20000 \\
\fourhops &5000 &5000 &10000 &20000 \\
\threehopl &10000 &10000 &- &20000 \\
\fourhopl &10000 &10000 &20000 &40000 \\
\bottomrule
\end{tabular}
\caption{Training steps for each training stage of curriculum learning}\label{tab: cl_stage}
\end{table}

\section{Detailed results} \label{app:detailed_results}

\subsection{Results for LMs on \hops} \label{app:detailed_results:hops}
We plot the test accuracy of LMs on \hops\ across training steps in Figure \ref{fig:small_model_comparison}.
The pattern is similar to the one observed in Figure \ref{fig:large_model_comparison}. 
Models trained with small budgets only give modest improvement over random baseline (i.e.\ $2\%$ for \hops). 
Larger budgets not only lead to higher accuracy, but also achieves this with much less training steps.

We also report \hops\ results of models trained with mixed learning and curriculum learning in Figure \ref{fig:small_mix_cl}.
Still,  we observe that curriculum learning gives the best result compared to the baseline and mixed learning.

\subsection{Standard deviation}

For each experiment reported in Section \ref{sec:e1} and \ref{sec:e3}, we made 3 runs based on different random seeds.
We report the mean and standard deviation of the test accuracy for each model in Table \ref{tab:detailed_results}.
For most results we do not observe a large standard deviation, indicating that our conclusion is robust to the randomness.
For particular runs there is a large deviation, especially when the data budget is not enough (e.g.\ model trained with curriculum learning on \fourhopl with $\times2$ budget), which gets smaller when we further add more data into the training set. 

\subsection{Log scale of data budget}

We plot the minimal data budget required to solve $k$-hop tasks on a log scale as $k$ increases. 
The data points are based on numbers in Table \ref{tab:baseline}. 
Figure \ref{fig:logscale} shows the results, confirming that the required data budget grows exponentially with $k$.

\begin{figure}
    \centering
    \includegraphics[width=\linewidth]{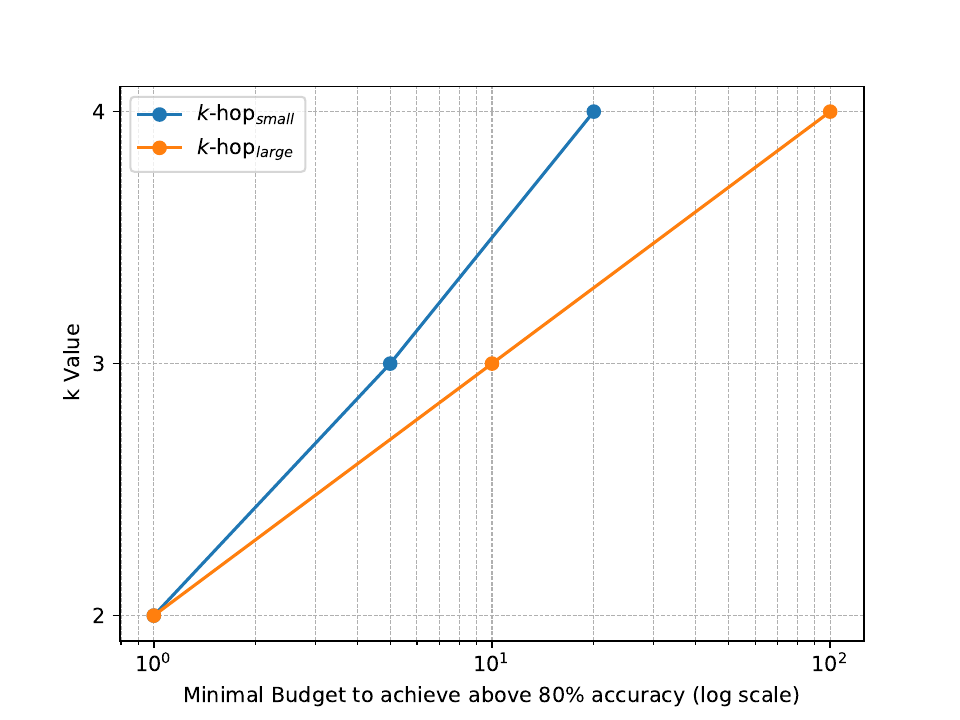}
    \caption{Minimal data budget to solve $k$-hop tasks.}
    \label{fig:logscale}
\end{figure}

\begin{figure*}[t]
    \centering
    \includegraphics[width=1.0\linewidth]{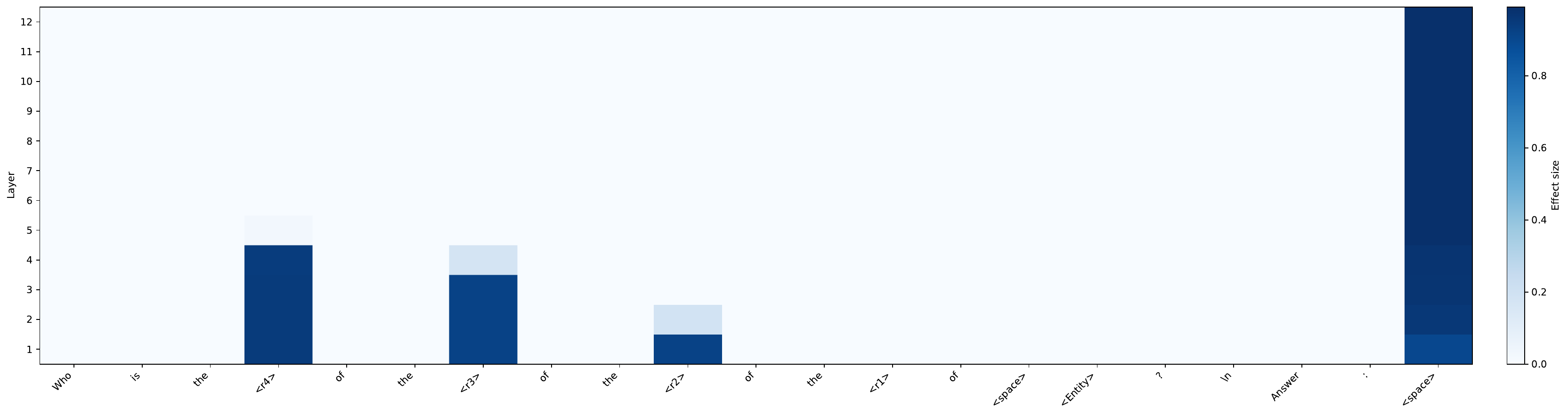}
    \caption{Results for activation patching replacing the residual stream of a particular layer across prompt tokens.}
    \label{fig:patching_across_tokens}
\end{figure*}

\begin{figure*}[t]
    \centering
    \includegraphics[width=1.0\linewidth]{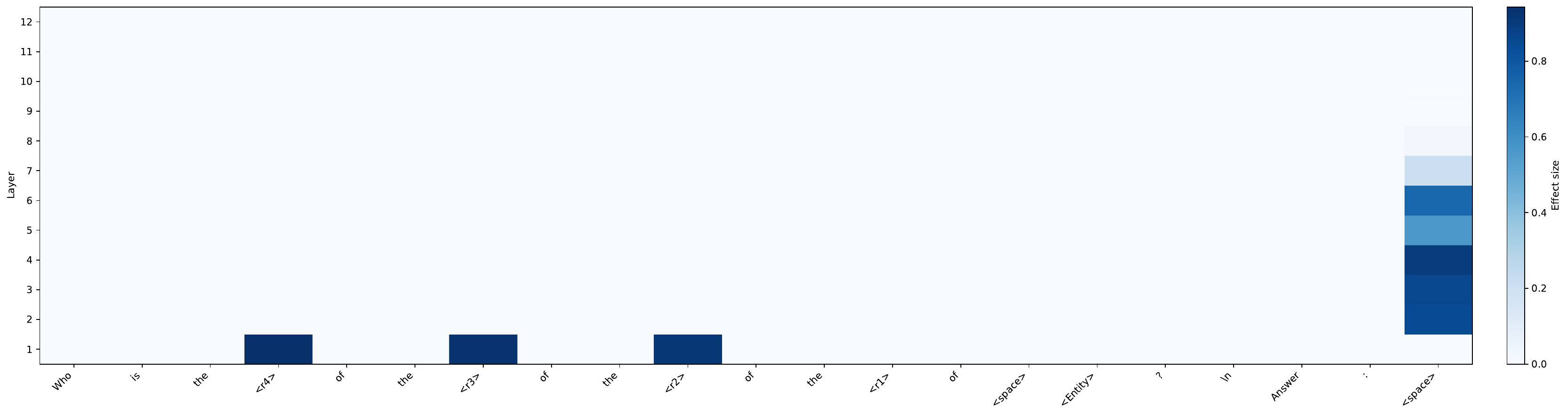}
    \caption{Results for activation patching replacing the MLP output of a particular layer across prompt tokens.}
    \label{fig:patching_across_tokens_mlp}
\end{figure*}

\section{Additional mechanistic interpretability experiments} \label{app:mi}

\subsection{Patching preceding prompt tokens} \label{app:mi:preceding_tokens}
Figure \ref{fig:probe} suggests that only the last token (e.g.\ the whitespace \textit{<space>}) includes information about all bridge entities, and hence the reasoning process likely occurs at this position. 
In this section, we use activation patching to further demonstrate that the reasoning process of our language model only occurs at the last token position instead of preceding prompt tokens.

Our activation patching still addresses three types of runs: clean run, corrupted run and patched run. 
For each clean run, we randomly select a distinct instance as the corrupted run. 
For each layer and each token position in the input prompt, we create a patched run by replacing the residual stream of the clean run with that of the corrupted run at the corresponding position. 
The causal effect is calculated as $P_{\text{clean}} - P_{\text{patched}}$, where $P_{\text{clean}}$ denotes the output probability of the correct answer in the clean run, and $P_{\text{patched}}$ denotes the probability in the patched run. 
We report the average causal effect over 1000 held-out instances.

Figure \ref{fig:patching_across_tokens} presents the activation patching results across token positions.  
Noticeably, no significant causal effects are observed in any token positions following the \textit{<Entity>} token, except for the last \textit{<space>} token. 
Since the \textit{<Entity>} token is the first position where the model can access complete query information (i.e., relations and source entity), this result supports our claim that the reasoning process primarily occurs at the last token position.

We also observe large causal effects on relation tokens when patching deeper layers (e.g., the $4$th layer for the \textit{<r4>} token). 
We consider this effect is because the model only start to read the information of \textit{<r4>} relation since the $4$th layer when predicting the answer.
Hence, deeper layers of \textit{<r4>} position should not involve any reasoning-related computation.
To show this, we also perform the same activation patching experiment by replacing only the output of each MLP layer. 
As shown in Figure \ref{fig:patching_across_tokens_mlp}, the relation tokens only show causal effects in the first layer, further supporting our hypothesis that deeper layers do not reprocess relation information.

\subsection{Causal effects across training steps} \label{app:mi:cl}
In Section \ref{sec:e2}, we observed that LMs learn to solve \hop\ tasks through a layer-wise lookup process, with specific layers responsible for producing bridge entities from $1$-hop to $k$-hop. 
A key question is whether these circuits (i.e., layers) are developed sequentially from $1$-hop to $k$-hop or simultaneously across multiple hop positions during training. 
To investigate this, we apply the activation patching experiment described in Section \ref{sec:e2} at every checkpoint of the training process.

We focus on the model trained on \fourhopl\ with a $\times100$ budget, following the setup in Section \ref{sec:e2}. 
Checkpoints are saved every 1k training steps, and we apply activation patching at the last input token position. 
For each checkpoint, we measure the causal effect of each layer for bridge entities at each hop position.

\paragraph{LMs tend to build circuits of different $i$-hop bridge entities simultaneously.}
Figure \ref{fig:patch_across_training} shows the causal effect of each layer across training steps. 
We observe that circuits responsible for $1$-hop, $2$-hop, and $3$-hop bridge entities emerge simultaneously at around the $17000$th training step, with each circuit appearing in distinct layers (e.g., the $1$st layer for $1$-hop entity).
This pattern indicates that the model tends to develop circuits for different hop positions at once rather than sequentially from easier (e.g., $1$-hop) to more complex (e.g., $3$-hop) entities.

\paragraph{Curriculum learning gradually build circuits on existing ones.} 
We further analyze the development of circuits in the curriculum learning model trained on the \fourhop\ task with a $\times5$ budget (Section \ref{sec:e3}). 
Training this model includes 3 stages.
Checkpoints are saved every 1k steps, and causal effects are calculated at each stage.
For each stage, we calculate the causal effects using the following corrupted runs:  
\begin{itemize}
    \item \textit{C}\textsubscript{1-hop}: Assesses $1$-hop circuits across stage 1, 2 and 3.  
    \item \textit{C}\textsubscript{2-hop}: Assesses $2$-hop circuits across stages 2 and 3.  
    \item \textit{C}\textsubscript{3-hop}: Assesses $3$-hop circuits in stage 3.  
\end{itemize} 
\begin{figure*}[!t]
    \centering
    \includegraphics[width=\linewidth]{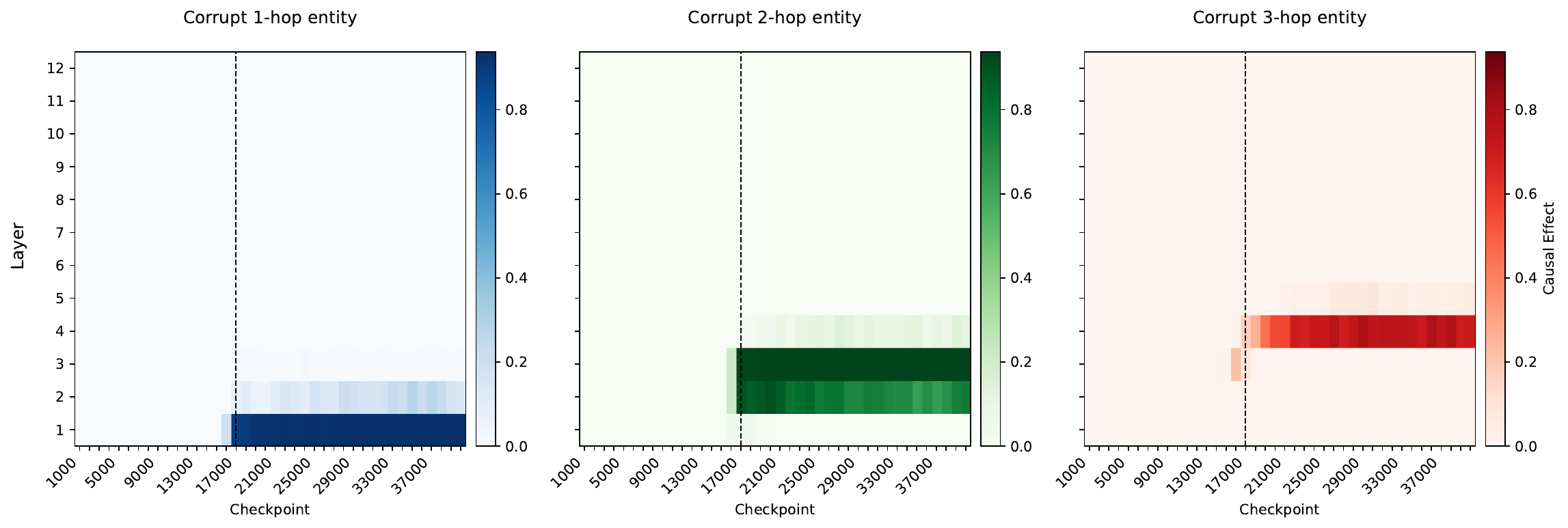}
    \caption{Causal effects calculated by corrupting $1$-hop, $2$-hop and $3$-hop bridge entity in our \textit{baseline} model.
    x-axis refers to checkpoints across training steps. 
    We observe that the circuits corresponding to different hop positions tend to emerge at once (e.g., around the $17{,}000$th step), rather than gradually developing over time.
    }
    \label{fig:patch_across_training}
\end{figure*}

\begin{figure*}[!t]
    \centering
    \includegraphics[width=\linewidth]{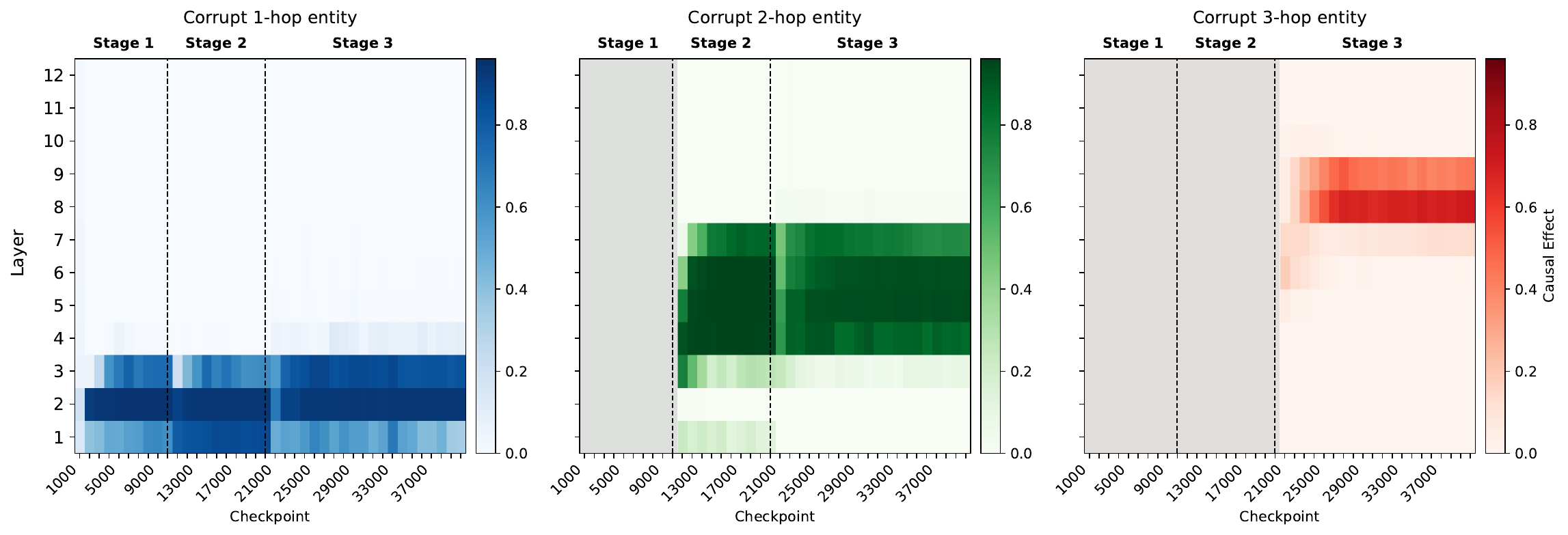}
    \caption{Causal effects calculated by corrupting $1$-hop, $2$-hop and $3$-hop bridge entity in our \textit{curriculum learning} model.
    x-axis refers to checkpoints across training steps.
    Gray regions indicate stages where causal effects are not calculated for certain entities, e.g., stage 1 does not include $3$-hop bridge entities in the training data, so the rightmost figure omits these effects in stage 1.   
    Circuits for higher-hop entities tend to be established on top of existing ones for lower-hop entities. 
    }\label{fig:cl_across_training}
\end{figure*}

Figure \ref{fig:cl_across_training} presents the results with our curriculum learning model. 
During stage 1, the model establishes circuits for $1$-hop entities. 
In stage 2, the $2$-hop circuit emerges, building upon the existing $1$-hop circuit. 
Stage 3 follows the same pattern, with the $3$-hop circuit extending the prior circuits. 
This layer-by-layer construction supports our hypothesis that curriculum learning encourages progressive circuit development, allowing higher-hop circuits to build upon existing lower-hop circuits, explaining the observed effectiveness in Section \ref{sec:e3}.

Curriculum learning has also been explored in prior work \citep{deng2024explicit, hao2024training}, where the focus is on internalizing explicit reasoning abilities. 
These studies start from chain-of-thought (CoT) rationales and train language models to reason with progressively fewer prompt tokens. 
In contrast, our setup does not rely on any explicit rationales. 
Instead, we study how curriculum learning affects the data budget required for training and provide an explanation for why such strategies improve sample efficiency from a mechanistic interpretability perspective.

\newcolumntype{L}{>{\raggedright\arraybackslash}X}
\newcolumntype{C}{>{\centering\arraybackslash}X}

\begin{table*}[!htp]
  \centering
  \scriptsize
  \setlength{\tabcolsep}{4pt}
  \begin{tabularx}{\textwidth}{@{} L L L | *{7}{C C} @{}}
    \toprule
    \multirow{2}{*}{Model}
      & \multirow{2}{*}{Size}
      & \multirow{2}{*}{Task}
      & \multicolumn{2}{c}{×1}
      & \multicolumn{2}{c}{×2}
      & \multicolumn{2}{c}{×5}
      & \multicolumn{2}{c}{×10}
      & \multicolumn{2}{c}{×20}
      & \multicolumn{2}{c}{×50}
      & \multicolumn{2}{c}{×100} \\
    & & 
      & Mean & Std 
      & Mean & Std 
      & Mean & Std 
      & Mean & Std 
      & Mean & Std 
      & Mean & Std 
      & Mean & Std \\
    \cmidrule(lr){1-3}
    \cmidrule(lr){4-5}\cmidrule(lr){6-7}\cmidrule(lr){8-9}
    \cmidrule(lr){10-11}\cmidrule(lr){12-13}\cmidrule(lr){14-15}\cmidrule(lr){16-17}

    \multirow{6}{*}{baseline}
      & \multirow{3}{*}{small}
        & 2-hop  & 99.8 & 0.1 &      &      &      &      &      &      &      &      &      &      &      &      \\
      &   & 3-hop  &  5.7 & 0.0 & 12.6 & 2.6 & 99.9 & 0.1 & 100.0 & 0.0 &      &      &      &      &      &      \\
      &   & 4-hop  &  4.3 & 0.4 &  6.2 & 0.5 &  6.7 & 0.2 &   9.2 & 0.4 &  96.4 & 3.1 &  96.4 & 0.0 & 100.0 & 0.0 \\
    \cmidrule(lr){2-17}
      & \multirow{3}{*}{large}
        & 2-hop  & 99.9 & 0.0 &      &      &      &      &      &      &      &      &      &      &      &      \\
      &   & 3-hop  &  2.5 & 0.3 &  3.1 & 0.1 &  4.9 & 1.6 &  94.6 & 9.4 & 100.0 & 0.0 &      &      &      &      \\
      &   & 4-hop  &  2.0 & 0.3 &  2.6 & 0.3 &  3.1 & 0.1 &   3.7 & 0.3 &   4.0 & 0.4 &   6.3 & 1.0 & 100.0 & 0.0 \\

    \midrule

    \multirow{6}{*}{mix}
      & \multirow{3}{*}{small}
        & 2-hop  &100.0 & 0.0 &      &      &      &      &      &      &      &      &      &      &      &      \\
      &   & 3-hop  & 29.2 & 3.0 & 88.1 & 8.6 & 99.9 & 0.1 & 100.0 & 0.0 &      &      &      &      &      &      \\
      &   & 4-hop  & 16.4 & 1.8 & 29.3 & 1.8 & 71.3 &41.4 &  99.8 & 0.1 & 100.0 & 0.0 & 100.0 & 0.0 & 100.0 & 0.0 \\
    \cmidrule(lr){2-17}
      & \multirow{3}{*}{large}
        & 2-hop  &100.0 & 0.0 &      &      &      &      &      &      &      &      &      &      &      &      \\
      &   & 3-hop  &  8.3 & 1.4 & 11.2 & 3.8 & 38.7 &18.1 & 100.0 & 0.0 & 100.0 & 0.0 &      &      &      &      \\
      &   & 4-hop  &  2.1 & 0.2 &  2.7 & 0.1 &  3.7 & 0.2 &   3.4 & 0.1 &   4.3 & 0.6 &   7.2 & 1.9 & 100.0 & 0.0 \\

    \midrule

    \multirow{6}{*}{curriculum}
      & \multirow{3}{*}{small}
        & 2-hop  &100.0 & 0.0 &      &      &      &      &      &      &      &      &      &      &      &      \\
      &   & 3-hop  & 56.1 & 1.5 & 96.0 & 0.7 &100.0 & 0.0 & 100.0 & 0.0 &      &      &      &      &      &      \\
      &   & 4-hop  & 29.3 & 2.7 & 68.7 & 5.4 & 99.6 & 0.2 & 100.0 & 0.0 & 100.0 & 0.0 & 100.0 & 0.0 & 100.0 & 0.0 \\
    \cmidrule(lr){2-17}
      & \multirow{3}{*}{large}
        & 2-hop  &100.0 & 0.0 &      &      &      &      &      &      &      &      &      &      &      &      \\
      &   & 3-hop  & 35.3 & 1.5 & 96.3 & 1.2 &100.0 & 0.0 & 100.0 & 0.0 & 100.0 & 0.0 &      &      &      &      \\
      &   & 4-hop  &  9.4 & 1.9 & 36.1 &14.8 &100.0 & 0.0 & 100.0 & 0.0 & 100.0 & 0.0 & 100.0 & 0.0 & 100.0 & 0.0 \\

    \bottomrule
  \end{tabularx}
  \caption{Accuracy (mean ± std) for 2-/3-/4-hop tasks under varying data budgets. Blank cells denote that the data budget exceeds the number of available questions.}
  \label{tab:detailed_results}
\end{table*}

\end{document}